\pdfoutput=1

\documentclass[format=acmsmall, review=false, screen=true]{acmart}

\usepackage{booktabs} 

\usepackage{framed}
\usepackage{hyperref}
\usepackage{balance}
\usepackage{natbib}
\usepackage{epsfig}
\usepackage{color}
\usepackage{subfigure}
\usepackage{multirow,tabularx}

\usepackage{mathtools}
\usepackage{graphicx}
\usepackage{epstopdf}
\usepackage{bm}

\usepackage{csquotes}
\usepackage{array}
\usepackage[yyyymmdd,hhmmss]{datetime}
\usepackage[subject={Todo}]{pdfcomment}
\usepackage[textsize=scriptsize,bordercolor=black!20]{todonotes}
\usepackage{xcolor,colortbl}
\usepackage{pifont}
\usepackage[algo2e,titlenumbered,ruled]{algorithm2e} 
\usepackage{lipsum,environ,amsmath}
\usepackage{slashbox,booktabs,amsmath}

\usepackage{xspace}

\newcounter{Lcount}
\newcommand{\numsquishlist}{
   \begin{list}{\arabic{Lcount}. }
    { \usecounter{Lcount}
 \setlength{\itemsep}{-.1ex}      \setlength{\parsep}{0ex}
      \setlength{\topsep}{0ex}       \setlength{\partopsep}{0ex}
      \setlength{\leftmargin}{1em} \setlength{\labelwidth}{1em}
      \setlength{\labelsep}{0.1em} } }
\newcommand{\numsquishend}{\end{list}}

\newcommand{\squishlist}{
   \begin{list}{$\bullet$}
    { \setlength{\itemsep}{-.1ex}      \setlength{\parsep}{0ex}
      \setlength{\topsep}{0ex}       \setlength{\partopsep}{0ex}
      \setlength{\leftmargin}{.8em} \setlength{\labelwidth}{1em}
      \setlength{\labelsep}{0.5em} } }
\newcommand{\squishend}{\end{list}}

\definecolor{Gray}{gray}{0.85}
\definecolor{LightGray}{rgb}{0.9,0.9,0.9}
\definecolor{LightBlue}{rgb}{0.8,0.8,1}

\newcommand{\indep}{\rotatebox[origin=c]{90}{$\models$}}

\newcounter{problem}
\newenvironment{problem}[1][htb]
  {
   \begin{algorithm2e}[#1]%
   \SetAlFnt{\small}
    \SetAlCapFnt{\small}
    \SetAlCapNameFnt{\small}
    \SetAlCapHSkip{0pt}
  }{\end{algorithm2e}}

\def\BibTeX{{\rm B\kern-.05em{\sc i\kern-.025em b}\kern-.08emT\kern-.1667em\lower.7ex\hbox{E}\kern-.125emX}}
    
\copyrightyear{2021}
\acmYear{2021}
\setcopyright{rightsretained}
\acmJournal{TKDD}
\acmVolume{15} \acmNumber{2} \acmArticle{15} \acmMonth{1} \acmPrice{}
\acmDOI{10.1145/3424670}
\acmISBN{978-1-4503-9999-9/18/06}

\begin{document}

%
\title{Identifying Linear Models in Multi-Resolution Population Data using Minimum Description Length Principle to Predict Household Income}

%
\author{Chainarong Amornbunchornvej}
\email{chainarong.amo@nectec.or.th}
\orcid{0000-0003-3131-0370}
\authornote{Corresponding author}
\author{Navaporn Surasvadi}
\email{navaporn.surasvadi@nectec.or.th}
\orcid{1234-5678-9012}
\author{Anon Plangprasopchok}
\email{anon.plangprasopchok@nectec.or.th}
\orcid{1234-5678-9012}
\author{Suttipong Thajchayapong}
\email{suttipong.thajchayapong@nectec.or.th}
\orcid{1234-5678-9012}
\affiliation{%
  \institution{Thailand's National Electronics and Computer Technology Center (NECTEC)}
  \city{Pathum Thani}
  \country{Thailand}
} %

%
\renewcommand{\shortauthors}{C. Amornbunchornvej, et al.}

%
\begin{abstract}
One shirt size cannot fit everybody, while we cannot make a unique shirt that fits perfectly for everyone because of resource limitation. This analogy is true for the policy making. Policy makers cannot establish a single policy to solve all problems for all regions because each region has its own unique issue.  In the other extreme, policy makers also cannot create a policy for each small village due to the resource limitation. Would it be better if we can find a set of largest regions such that the population of each region within this set has common issues and we can establish a single policy for them? In this work, we propose a framework using regression analysis and minimum description length (MDL) to find a set of largest areas that have common indicators, which can be used to predict household incomes efficiently. Given a set of household features, and a multi-resolution partition that represents administrative divisions, our framework reports a set $\mathcal{C}^*$ of largest subdivisions that have a common predictive model for population-income prediction. We formalize a problem of finding $\mathcal{C}^*$ and propose the algorithm that can find $\mathcal{C}^*$ correctly. We use both simulation datasets as well as a real-world dataset of Thailand's population household information to demonstrate our framework performance and application. The results show that our framework performance is better than the baseline methods. Moreover, we demonstrate that the results of our method can be used to find indicators of income prediction for many areas in Thailand. By adjusting these indicator values via policies, we expect people in these areas to gain more incomes. Hence, the policy makers are able to plan to establish the policies by using these indicators in our results as a guideline to solve low-income issues. Our framework can be used to support policy makers to establish policies regarding any other dependent variable beyond incomes in order to combat poverty  and other issues. We provide the R package, MRReg, which is the implementation of our framework in R language. MRReg package comes with the documentation for anyone who is interested in analyzing linear regression on multiresolution population data.

\end{abstract}

%
%


%

\renewcommand{\shortauthors}{C. Amornbunchornvej et al.}

\maketitle

\section{Introduction}
Ending poverty in all its form everywhere has been recognized as the greatest global challenge in the 2030 Agenda for Sustainable Development~\cite{assembly20152030}.  Annually, there are at least 18 millions human deaths that caused by poverty~\cite{pogge2005world}.  A common goal is to increase income beyond poverty line (e.g. 1.90 USD per day). However, poverty alleviation often requires comprehensive measures. In practice, implementing poverty alleviation programs would depend on the realities, capabilities and level of development of each nation. 
In the past, there was a tendency to use a one-size-fit-all policy as a solution in the international level~\cite{lahn2017curse}.
Even though one-size-fit-all policies enable uncomplicated deployment of government resources, in reality, they are not suitable in alleviating poverty as each region possesses different problems and socioeconomic characteristics~\cite{berdegue2002rural}. The solution in one region might not work for another region even if they share some properties~\cite{lahn2017curse,RIOJA2004429}. If one were to drill down to the household-level poverty, one would often find that each household faces different problems in respect to demographic, health, education, living standards, and accessibility to available resources. Subsequently, lifting those household members above the poverty line in sustainable way would often require targeted poverty alleviation with unique solutions. However, in reality, many developing countries often do not possess enough resources. Hence, it is challenging for policy makers to be able to optimize between cheap one-size-fit-all and expensive targeted individual-based approaches.  



\section{Related work}
To combat poverty, one should find root causes of issues and all related factors that contribute to poverty of populations.  There are many factors that can contribute to poverty,  such as  health issues, lacking of education, inaccessibility of public service, severe living condition, lacking of job opportunities, etc ~\cite{Strotmann2018,alkire2019global}. To measure the degree of poverty, several works proposed to use the poverty indices, such as Human Poverty Index (HPI)~\cite{anand1997concepts}, Multidimensional Poverty Index (MPI)~\cite{alkire2010multidimensional,alkire2019global}, etc. However, a specific poverty index covers only some aspects of population, which might not include the actual root cause predictors~\cite{alkire2010multidimensional,Strotmann2018}. Additionally, different areas typically have different issues and root causes of poverty~\cite{commins2004poverty,pringle2000cross}. 

To infer which factors are the main causes of poverty in the area from population data, the first step is to find the associations between a dependent variable (e.g. poverty index), and independent variables (potential predictors). This is because the causal inference requires association relations among variables to be known before we can find causal directions~\cite{pearl2009causality,peters2017elements}. One of the approaches that is widely used in social science to find association relations is Regression Analysis~\cite{nathans2012interpreting}. It has been used as a part of framework in poverty indices analysis by the works in~\cite{njuguna2017constructing,jean2016combining}. Linear regression has been extended and used in various types of data beyond linear-static data, such as time series data~\cite{10.2307/1912791,udny1927method,atukeren2010relationship}, non-linear data~\cite{an2007fast,kostopoulos2018semi}, etc.

In this work, we are interested to find variables that associate with population household incomes, which is one of the main components in several poverty indices~\cite{anand1997concepts,alkire2010multidimensional,alkire2019global}. 

However, different areas resolutions (e.g. provinces vs. country levels) might have different predictors that are associated with household incomes. In the perspective of policy makers, placing the wrong policies to the areas might not be able to solve the issues. For example, urban and rural areas have different issues of poverty~\cite{commins2004poverty}. Hence, it is important to find the way to select the model from the right resolutions that truly represents issues of areas. According to the Occam's razor principle, we prefer a policy that is simple and covers issues of larger population as an efficient policy. One of the model selection approaches that is widely used is Minimum Description Length (MDL)~\cite{hansen2001model}, which was introduced in the field of data analytics by \cite{RISSANEN1978465}. It is used in linear regression setting as a feature selection criteria to find the informative features~\cite{schmidt2012consistency}. Nevertheless, to the best of our knowledge, there is no framework that compares the regression models from different areas resolution using MDL and asks which one is a better model as a predictor for household-income prediction.

\begin{figure}
    \centering
    \includegraphics[width=1\columnwidth]{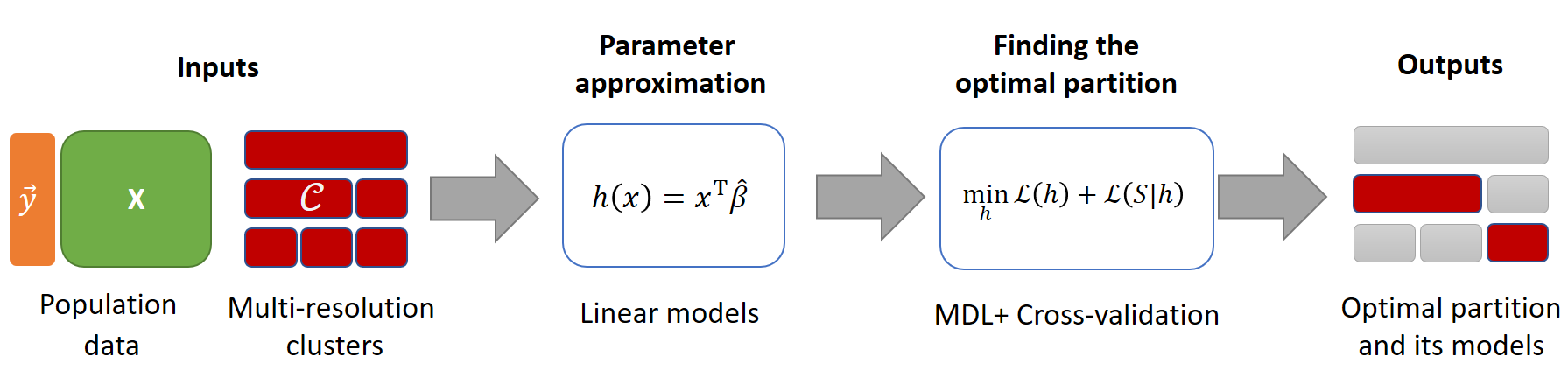}
    \caption{The proposed framework overview. Given a dataset $S=\{\vec{y},\textbf{X}\}$ and a set of Multi-resolution clusters $\mathcal{C}$ (Def.~\ref{def:MRCs}) as inputs  where a vector $\vec{y}$ contains realizations of dependent variables, and an $n\times l$ matrix $\textbf{X}$ contains realizations of independent variables.   In the first step, the framework performs the parameter estimation (Section~\ref{sec:paraest}). Then, we use MDL (Section~\ref{sec:mdl}) and cross-validation scheme (Section~\ref{sec:crossv}) to infer the maximal homogeneous partition in Section~\ref{sec:inferMaxPar}, which is the optimal-resolution partition that effectively represents population. }
    \label{fig:mainFramework}
\end{figure}

\subsection{Comparison with the clustering problem} 
Given a set of multi-resolution partitions as an input of our framework along with the population data (realizations of dependent and independent variables).  We propose the framework that is able to infer a subset of multi-resolution partitions (Fig.~\ref{fig:mainFramework}), which is  a set of largest partition s.t. each partition has a model that represents population well in the term of predictive performance (see Section~\ref{sec:crossv}). In contrast, in the clustering problem, the main goal is to find a set of clusters or partitions that is fitting the population data well w.r.t. some similarity function. Typically, the inputs of clustering are the number of clusters and/or some threshold(s) that define clusters, as well as population data. 

Hence, in a dataset that has only population data without  a set of multi-resolution partitions, we can use some data clustering approach to infer the hierarchical clusters (multi-resolution partitions). Then, our proposed framework can be used to find the optimal partition in the hierarchical clusters that covers the entire population s.t. the partition models represent the population well in the aspect of prediction of the target dependent variable. 

Policy makers might consider each cluster in the optimal partition from our framework as a group of individuals that shares the same indicators that can predict the target variable well.  If policy makers want to change the value of target variable in the population, then they might establish a policy for each cluster in our optimal partition instead of either making a single one-size-fit-all policy or having a single unique policy for each individual.

\begin{table}[]
\caption{Comparison between Mixture of Regression and our proposed framework. }
\label{TB:MixtureComp}
\begin{tabular}{|c|c|c|}
\hline
\rowcolor[HTML]{ECF4FF} 
\textbf{Properties\textbackslash{}Frameworks} & \textbf{Mixture of Regression} & \textbf{Proposed framework} \\ \hline
Assign individuals into partitions            & Yes                             & Yes                         \\ \hline
\rowcolor[HTML]{EFEFEF} 
Input: multi-resolution partitions            & Not required                     & Required                     \\ \hline
Input: number of clusters                     & Required                         & Not required                 \\ \hline
\rowcolor[HTML]{EFEFEF} 
Homogeneity of model of clusters               & Not report                      & Report                      \\ \hline
Time complexity of exact solution                               & No upper bound                  & Have upper bound            \\ \hline
\end{tabular}
\end{table}
\subsection{Comparison with mixture models} 
The concept of mixture models has been widely used for modeling complex distributions by using several mixture components for representing ``subdistribution'' or ``subpopulations''~\cite{MCLA2000}. 

Generally, an individual subdistribution is modeled using a ``simpler'' distribution, e.g., Gaussian distribution. Given a specific number of mixture components and a distribution type of mixture components, a finite mixture model will estimate degree of membership of an individual datapoint to all mixture components as well as parameters of the components. Data points belong to the same partition if they have the highest membership degree to the same mixture component. Mixture models can be generalized for regression problems~\cite{deVeaux:1989:MLR,Wedel1995}. Particularly, a regression function is used, instead of a distribution function, for representing a mixture component. Datapoints belong to the same component would share similar ``behavior'' according to the regression function and its variables. 

In contrary to traditional data clustering techniques, a mixture model can form a data partition without any similarity function defined. The model, however, requires an assumption of subpopulations that are distributed or characterized according to a certain class of distribution or regression function. Ideally, the mixture model framework would naturally fit to our problem if a set of policies is intended to apply to subpopulations, which can be clearly defined by a certain class of distribution or a regression function. Since it is not straightforward to impose structural constraints such as administrative regions (e.g., a pair of administrative regions at the same level cannot belong to the same subpopulation) into the inference algorithm, one can simply let the mixture model to flexibly ``search'' for subpopulation partitions and subsequently impose the administrative constraints if some partition violates such constraints. Our approach, however, explicitly takes the structural constraints into account by ``searching'' and evaluating the data partition along the administrative hierarchy. Consequently, the data partition from our approach always complies to the constraints and requires no additional step -- the constraint checking -- as in the mixture model case.

Table~\ref{TB:MixtureComp} briefly illustrates the difference between Mixture of Regression and our proposed framework. Both approaches assign individuals into a partition. While Mixture of Regression requires the number of clusters (components) as a parameter, our approach uses the multi-resolution partitions as an input. Additionally,  instead of only assigning individuals into a partition, our approach can also evaluate how strong the individuals in the same cluster share the same model (see Section~\ref{sec:crossv}). Moreover, the time complexity of our approach is bounded for the exact solution, while the bound of Mixture approach is unknown; but there is an algorithm~\cite{li2018learning} that can provide the approximate solution with the bound $\mathcal{O}(Nd)$ where $N$ is a number of individuals and $d$ is a number of dimensions. 

\subsection{Our contributions}

To fill the gaps mentioned above, in this paper, we aim to formalize the multi-resolution model selection problem in regression analysis using MDL principle and propose the framework (Fig.~\ref{fig:mainFramework}) as a solution of our model selection problem. The proposed framework is capable of:

\squishlist
\item {\bf Inferring the optimal-resolution partition:} inferring the best fit partition that effectively represents population in the aspect of model prediction and MDL; and
\item {\bf Inferring the optimal model for each area:} inferring the best fit model among candidate functions given an area population data as well as choices of functions in a linear function class.
\squishend
After formalizing the new problem, we provide the algorithm that guarantees soundness and completeness of finding the optimal partition, which is the output of the new problem. We evaluate our approach performance and demonstrate its application using both simulation and real-world datasets of Thailand population households.  Our framework can be generalized beyond the context of poverty analysis. 
\newcommand{\argmax}{\mathop{\mathrm{argmax}}\limits} 
\newcommand{\argmin}{\mathop{\mathrm{argmin}}\limits} 
\newtheorem{assumption}{Assumption}
\section{Problem formalization}

\begin{table}
\caption{Notations and symbols }
\label{tb:symbloTable}
\begin{small}
\begin{tabular}{|l|p{4in}|}
\hline
{\bf Term and notation} & {\bf Description}                                                                                                                                \\ \hline\hline
 $n$ & Number of individuals \\ \hline
  $d$ & Number of independent variables \\ \hline
  $l$ & Number of parameters \\ \hline
 $D=\{1,\dots,n\}$               &  Set of individual indices           \\ \hline
  $C$               &  Cluster where $C \subseteq D$           \\ \hline
 $\mathcal{C}=\{C_{j,k}\}$         & Set of Multi-resolution clusters where $C_{j,k}$ is a $j$th cluster at $k$th layer. \\ \hline
  $\mathcal{C}^*=\{C\}$ &  Maximal homogeneous partition \\ \hline
$X,x$    & Independent variable and its realization \\ \hline
$Y,y$    & Dependent variable and its realization \\ \hline
$S=\{(x_1,y_1),\dots,(x_n,y_n)\}$ & Dataset contains realizations of $X$ and $Y$ \\ \hline
 $\mathcal{L}(h)$ & Number of bits of representation of a model $h$ \\ \hline
 $\mathcal{L}(S|h)$ & Number of bits of representation of dataset $S$ given a model $h$ \\ \hline
  $\mathcal{L}_{\mathbb{R}}(x)$ & Number of bits of representation of real-number vector $x$ \\ \hline
   $\gamma$ & Homogeneity threshold \\ \hline
   $H,\mathcal{H},\mathbb{H}$ & Function class, set of models, and hypothesis space \\ \hline
   $\mathcal{I}(\mathcal{C}',H_1,H_2)$ & Model Information Reduction Ratio where $\mathcal{C}$ is a partition of $D$\\ \hline
   $\mathcal{I}(\mathcal{C}'_1,\mathcal{C}'_2,H)$ & Cluster Information Reduction Ratio  where $\mathcal{C}'_1,\mathcal{C}'_2$ are partitions of $D$ \\ \hline
\end{tabular}
\end{small}
\end{table}

Given a vector $\vec{y}=(y_1,\dots,y_n)$, an $n\times d$ matrix $\textbf{X}=(x_1,\dots,x_n)^\text{T}$, and a set of Multi-resolution clusters $\mathcal{C}=\{C_{j,k}\}$ (Def.~\ref{def:MRCs}).  $y_i \in \mathbb{R}$ is a realization of dependent random variable $Y$ of individual $i$,  $\textbf{X}(i,j) \in \mathbb{R}$ is a realization of  independent random variable $X_j$ of individual $i$, and $C_{j,k}$ is a $j$ cluster in $k$th layer of multi-resolution cluster set where individual $i$ is a member of $C_{j,k}$. In this work, the main purpose is to solve Problem~\ref{prob2} where a hypothesis space $H$ is  a linear function class.  The details of definitions and the problem formalization can be found the this section. 

Specifically, in Section~\ref{sec:PF-MDL}, we explore the concept of Minimum Description Length (MDL). Then, in Section~\ref{sec:modelsel}, we provide definitions as well as the formalization of Problem~\ref{prob2}, which is the main problem we aim to solve in this paper. In Section~\ref{sec:linprop}, we provide the properties of linear model in the MDL setting. We also provide the Table~\ref{tb:symbloTable} as a reference for the notations and symbols used in this paper.
\\
\subsection{Minimum description length on regression models}
\label{sec:PF-MDL}
Given a set of observation data below:

\begin{equation}
S=\{(x_1,y_1),\dots,(x_n,y_n)\}.
\end{equation}

Where $(x_i,y_i)$ are realizations of random variables $(X_i,Y_i)$ s.t. $x_i\in \mathbb{R}^d$ and $y_i \in \mathbb{R}$. The random variables $X_1,\dots,X_n$ are i.i.d. with an unknown distribution $P_X$ or $X_i\sim P_X$. The random variable $Y_i$ has a relation with $X_i$ below:

\begin{equation}
Y_i = h^*(X_i) + \epsilon.
\end{equation}

Where $h^*(X_i)$ is an unknown function that is a member of some hypothesis space $\mathbb{H}$ and $\epsilon$ is a noisy constant s.t. $X_i$ and $\epsilon$ are statistically independent or $X_i \indep \epsilon$.

In MDL, we would like to find the number of bits of the shortest representation that we can have for any given data based on some hypothesis space.


\begin{equation}
\label{eq:firstMDL}
\mathcal{L}(S)_{\mathbb{H}} =\min_{h\in\mathbb{H}} \mathcal{L}(h)+ \mathcal{L}(S|h)
\end{equation}

Where $\mathcal{L}(h)$ is a number of bits of representation for a predictive model $h$ and $\mathcal{L}(S|h)$  is a number of bits of representation for data $S$ given the predictive model $h$. In our case, $\mathcal{L}(h)$ is the number of bits we encode parameters and related information of function $h$. The $\mathcal{L}(S|h)$ is the number of bits we encode $S$ using $h$.

\begin{equation}
\label{eq:Lh}
\mathcal{L}(h) = \sum_{j=1}^l \mathcal{L}_{\mathbb{R}}(p_j)  
\end{equation}
\begin{equation}
\label{eq:LSh}
\mathcal{L}(S|h) = \sum_{i=1}^n \mathcal{L}_{\mathbb{R}}(x_i) +  \sum_{i=1}^n \mathcal{L}_{\mathbb{R}}( y_i-h(x_i) ) 
\end{equation}

Where $l$ is a number of parameters of $h$, $p_j$ is $j$th parameter of $h$, $\mathcal{L}_{\mathbb{R}}(x)$ is a function that returns a number of bits we need to encode a real number $x$, and $n$ is a number of data points in $S$. 

The model complexity in Eq.~\ref{eq:firstMDL} can be controlled by $\mathcal{L}(h)$. If two models perform equally, then the Eq.~\ref{eq:firstMDL} chooses the less complex model that has lower $\mathcal{L}(h)$ as a solution.  See Section~\ref{sec:ex2nonlin} for more discussion regarding how model complexity can be implemented and used to compared functions from difference classes. 

In the sparsity regularization aspect, we can modify $\mathcal{L}(h)$ to add the higher bits for a model representation (penalty) when the model $h$ has a higher number of parameters (nonzero coefficients). See Section~\ref{sec:SparsityReg} for more discussion.

In this paper, we the assumption below: 

\begin{assumption}
\label{assump:bignumberMoreCode}
For any $x_1,x_2 \in \mathbb{R}$ s.t. $|x_1|\geq|x_2|$, we have $\mathcal{L}_{\mathbb{R}}(x_1)\geq \mathcal{L}_{\mathbb{R}}(x_2)$.
\end{assumption}
 We have Assumption~\ref{assump:bignumberMoreCode} to represent the fact that the larger number contains similar amount of or more information that we need to encode. In the function $\mathcal{L}(S|h)$, there are two parts: the part that we encode $x$ or $\sum_{i=1}^n \mathcal{L}_{\mathbb{R}}(x_i)$ and the part that we encode $y$ by $h\in\mathbb{H}$ or $\sum_{i=1}^n \mathcal{L}_{\mathbb{R}}( y_i-h(x_i) )$. If $h$ predicts $y_i$ using $x_i$ perfectly, then we expect to encode no information of $y_i$, which makes $\mathcal{L}_{\mathbb{R}}( y_i-h(x_i) ) =0$ bits. However, if $h'\in\mathbb{H}$ makes larger error of predicting $y_i$ from $x_i$ than $h$, we expect to use more bits of encoding for $y_i$ using $h'$ and $x_i$ than $h$. Hence, Assumption~\ref{assump:bignumberMoreCode} naturally represents this intuitive notion of having more bits of encoding for a function $h'$ that makes more error of predicting $y_i$ from $x_i$.   
 
Now, we are ready to formalize MDL-function inference problem.

\begin{problem}[h!]
    \SetKwInOut{Input}{Input}
    \SetKwInOut{Output}{Output}
    \Input{A set  $S = \{(x_1,y_1), \dots, (x_n,y_n)\}$, and a hypothesis space $\mathbb{H}$ }
    \Output{An optimal MDL-function $h^*= \argmin_{h\in\mathbb{H}} \mathcal{L}(h)+ \mathcal{L}(S|h)$ }
    \caption{{MDL-function inference problem}}
	\label{prob1}
\end{problem}

\subsection{Model selection on multi-resolution clusters}
\label{sec:modelsel}

Given a set of individual indices $D=\{1,\dots,n\}$ where $i\in D$ refers to $i$th individual in a population, we can define a set of multi-resolution clusters below.

\begin{definition}[Multi-resolution cluster set (MRC set)]
\label{def:MRCs}
A set of multi-resolution clusters $\mathcal{C}=\{C_{j,k}\}$ of $n_c$ layers, where $k\in [1,n_c]$, and $C_{j,k}\subseteq D$ is a $j$th cluster at $k$th layer, has a following properties.
\squishlist
\item If $k>1$, then $\forall C_{j,k}, \exists C_{j',k-1}, C_{j,k} \subseteq C_{j',k-1}$.
\item All individuals belong to some cluster in each layer or $\forall k, \bigcup_j C_{j,k} = D$.
\item Clusters in the same layer are disjoint or $\forall k, \bigcap_j C_{j,k} = \emptyset$.
\squishend
\end{definition}

An example of multi-resolution cluster set is a national administrative division of regions in a country where the first layer is a national level, the second layer is a province level, and the last layer is a village level. Each cluster in each layer is a specific group of individuals that are governed by the same local government in a particular level. For instance, in the province level, each cluster $C$ represents citizens or households within a province $C$. Next, we have a concept of clusters that all members share the same joint distribution $P_{X,Y}$.

\begin{definition}[Homogeneous cluster]
\label{def:homocls}
Given a cluster $C \subseteq D$, and a set of population data $S_\mathcal{C}=\{(x_1,y_1),\dots,(x_n,y_n)\}$ where $x_i \in \mathbb{R}^d$ and $y_i\in\mathbb{R}$. A cluster $C$ is a homogeneous cluster of $S_C=\{(x_i,y_i)| i \in C\}$ where $x_i,y_i$ are realizations of $X_i,Y_i$ respectively, if the following conditions hold.
\numsquishlist
\item $\exists P_X, \forall i \in C, X_i \sim P_X$.
\item $Y_i = h^*(X_i) + \epsilon$ for some unknown function $h^*$.
\numsquishend
Where $\epsilon\in \mathbb{R}$ is a noise random variable from some unknown truncated distribution with the bound $[a,b]$ s.t. $\epsilon^*=max(|a|,|b|)$, $\mathbb{E}[\epsilon]=0$, and $|\epsilon| \leq \epsilon^*$. All subsets of $C$ are also homogeneous clusters.
\end{definition}

The concept of homogeneous cluster represents a specific area (e.g. province, village) that share a common relationship between a dependent variable and independent variables. For instance, a village $A$ has the health issue that affects household incomes of all villagers. Suppose $Y$ is a dependent variable of income and $X$ is an independent variable of degree of health-issue severeness. If a village $A$ is a homogeneous cluster, then $|y_i| \propto |x_i|$ for any household $i$ in this village. Because all individual households share the similar properties, policy makers should make a single set of policies to support the entire population within a homogeneous cluster.   

Next, we need the concept of a cluster set that covers all individuals. 
 
\begin{definition}[MRC Partition]
Given a multi-resolution cluster set $\mathcal{C}$. A set of clusters $\mathcal{C}_p$ is an MRC partition if $\mathcal{C}_p \subseteq \mathcal{C}$,  all clusters in $\mathcal{C}_p$ are disjoint, and $\bigcup_{C_i \in \mathcal{C}_p} C_i = D$.
\end{definition}

An MRC Partition represents a set of administrative divisions that covers all individuals within a country. A set of all provinces is an example of the MRC Partition.

However, we need the MRC Partition concept that can include the sub-areas from mixed layers of MRC set but still covers the entire population.  This is because some provinces might be homogeneous clusters that share the same properties, while, in other provinces, we need to break down to the village layer to get the homogeneous clusters. Our goal is to use the concept of MRC Partition to define a set of maximal homogeneous clusters (largest possible homogeneous cluster for a specific area) that covers the entire population. 

\begin{definition}[Maximal homogeneous cluster]

Given a multi-resolution cluster set $\mathcal{C}$, and a set of population data $S_\mathcal{C}=\{(x_1,y_1),\dots,(x_n,y_n)\}$ where $x_i \in \mathbb{R}^d$ and $y_i\in\mathbb{R}$. A cluster $C^*\in \mathcal{C}$ is a maximal homogeneous cluster if $C$ is a homogeneous cluster, and there is no other homogeneous cluster $C'$ s.t. $C'\in \mathcal{C}$ and $C^* \subset C'$.
\end{definition}

The numbers of elements within MRC Partition might represent a number of sets of policies that policy makers should consider to make in order to support each area common issues. Next, we need to define a concept to call datasets that have the maximal homogeneous partition, which is the MRC Partition that contains only maximal homogeneous clusters as members.

\begin{definition}[Multi-resolution-bivariate set (MRB)]
Let $\mathcal{C}$ be a multi-resolution cluster set, and $S_\mathcal{C}$ be a set of $\mathcal{C}$'s population data. $S_\mathcal{C}$ is a Multi-resolution-bivariate set if $S_\mathcal{C}$ satisfies the following properties.
\squishlist
\item There exists an MRC partition of homogeneous clusters, $\mathcal{C}^*=\{C\}$, s.t. for each $C\in \mathcal{C}^*$, $C$ is a maximal homogeneous cluster. We call $\mathcal{C}^*$ as the ``Maximal Homogeneous Partition''.
\item Each homogeneous $C_i \in \mathcal{C}^*$ has a separate function $h^*_i$ that might be different from any other function  $h^*_j$ of $C_j \in \mathcal{C}^*$ where $C_i\neq C_j$.
\squishend
\end{definition}

Suppose we have an MRB set $S_\mathcal{C}=\{(x_1,y_1),\dots,(x_n,y_n)\}$, and a function class $H$ (e.g. linear class, polynomial class, convex class). We can use the data in each cluster $C_{j,k}$ to estimate the parameters of function in $H$ to build a set of models $H_{j,k} \subset H$. Each model $h \in H_{j,k}$ has parameter values estimated from the data points $\{(x_i,y_i)\}$ in the cluster $C_{j,k}$.  Let $\mathcal{H}=\bigcup H_{j,k}$ be a set of models of $\mathcal{C}$, we can formalize our new computational problem: MDL-function inference on multi-resolution data problem, in Problem~\ref{prob2}.

\begin{problem}[h!]
    \SetKwInOut{Input}{Input}
    \SetKwInOut{Output}{Output}
    \Input{A set of individual indices $D=\{1,\dots,n\}$, a multi-resolution cluster set $\mathcal{C}=\{C_{j,k}\}$, an MRB set  $S_\mathcal{C} = \{(x_1,y_1), \dots, (x_n,y_n)\}$, and a function class $H$ }
    \Output{
    
    A set of models $\mathcal{H}$ w.r.t. $H$ and $\mathcal{C}$, a maximal homogeneous partition $\mathcal{C}^*=\{C\}$, and a set of optimal-MDL functions $H^*_{\mathcal{C}^*}=\{h^*\}$ where each $h^* \in H^*$ is the optimal MDL-function that belongs to a specific cluster $C\in \mathcal{C}^*$ using the hypothesis space as  $H$.
    
    }
    \caption{{MDL-function inference on multi-resolution data problem}}
	\label{prob2}
\end{problem}

\subsection{Linear function properties}
\label{sec:linprop}

Suppose we have $H$ as a linear function class, and $\mathcal{C}'$ as an MRC partition of $D$. We can have a measure of how well the models derived from $\mathcal{C}'$ represent the set of population data $S$ in the aspect of MDL below.

\begin{equation}
    \mathcal{L}(\mathcal{C}',H) = \sum_{C \in \mathcal{C}'} \big( \mathcal{L}(H)_{C} + \mathcal{L}(C|H) \big)  
    \label{func:MDLcls}
\end{equation}

\begin{equation}
    \mathcal{L}(H)_{C} =  \mathcal{L}_{\mathbb{R}}(\hat{\beta}_1^C)  +\mathcal{L}_{\mathbb{R}}(\hat{\beta}^C)
    \label{func:MDLcls2}
\end{equation}
\begin{equation}
    \mathcal{L}(C|H) =  \sum_{i \in C} \mathcal{L}_{\mathbb{R}}(x_i) +  \sum_{i \in C} \mathcal{L}_{\mathbb{R}}( y_i- \hat{\beta}_1^C-x_i^\text{T}\hat{\beta}^C )
    \label{func:MDLcls3}
\end{equation}
Where $\hat{\beta}_1^C \in \mathbb{R},\hat{\beta}^C \in \mathbb{R}^d$ are the estimators of linear function parameters based on the data points in $C$.  We have the next proposition to show that the sum of encoding bits of residuals from homogeneous clusters are less than the sum of encoding bits of residuals of function that its parameters are approximated by the union of these homogeneous clusters.

\begin{proposition}
\label{prop:homoAndHeteroCls}
Let $C=\{C_1,\dots,C_k\}$ be a set of homogeneous clusters of $S_\mathcal{C}$ but the cluster $C'=\bigcup_{C_i\in C}C_i$ is not a homogeneous cluster, and $H$ is a linear function class. Assuming that the functions that generate $Y$ from $X$ in $S_\mathcal{C}$ for each cluster $C_i \in C$ are in $H$.  

If $g'$ is a function that its parameters are approximated by $C'$ and $g_1\dots,g_k$ are functions that their parameters are approximated by $C_1,\dots,C_k$ respectively, where $g'$ and $g_1\dots,g_k$ are in $H$. We have $\sum_{i \in C'} \mathcal{L}_{\mathbb{R}}( y_i-g'(x_i)) \geq  \sum_{j=1}^k\sum_{i \in C_j} \mathcal{L}_{\mathbb{R}}( y_i-g_j(x_i))$.
\end{proposition}
\begin{proof}
Because each cluster in $C$ is homogeneous and the relation between $Y$ and $X$ is $Y_i\approx \beta_1+X_i^\text{T}\beta+\epsilon$. We can approximate $\beta_1,\beta$ using some estimator approach. Suppose $\hat{\beta}_1,\hat{\beta}$ are estimators for  $\beta_1,\beta$, we can have 

\begin{eqnarray*}
    \sum_{j=1}^k\sum_{i \in C_j} \mathcal{L}_{\mathbb{R}}( y_i-g_j(x_i)) = \sum_{j=1}^k\sum_{i \in C_j} \mathcal{L}_{\mathbb{R}}( y_i- \hat{\beta}_{1,j}-x_i^\text{T}\hat{\beta}_j)) \leq |C'|\mathcal{L}_{\mathbb{R}}(\epsilon^*).
\end{eqnarray*}

In contrast, if two clusters $C_1,C_2 \in C$ have a different functions $g_1,g_2$, for example,  $g_1(X)=-X^\text{T}\beta_c\ +\epsilon$ while $g_2(X)=-g_1(X)$, then by approximating $\beta_1,\beta$ for $g'(X)$ using data points from both $C_1,C_2$, the sum of residuals of data points from both clusters must be higher than $\epsilon^*$. Precisely, according to Def.~\ref{def:homocls}, $|\epsilon| \leq \epsilon^*$ for any homogeneous cluster. Suppose $g_1\neq g_2$ and $C_1,C_2$ are homogeneous clusters of $g_1$ and $g_2$ respectively. If we approximate the parameters from data points of both $C_1,C_2$, then we would have $g'$ that is not the same as both $g_1,g_2$ but $g'$ is something between $g_1,g_2$. For $g'$, the sum of encoding bits of residuals of $C_1$ is $\sum_{i \in C_1}\mathcal{L}_{\mathbb{R}}(y_i-g'(x_i))\leq  |C_1|\mathcal{L}_{\mathbb{R}}(\epsilon^*)+\sum_{i\in C_1}\mathcal{L}_{\mathbb{R}}(g'(x_i) - g_1(x_i))$ where the term $\sum_{i\in C_1}\mathcal{L}_{\mathbb{R}}(g'(x_i) - g_1(x_i))>0$ since $g'\neq g_1$. The same applies for $C_2$. 

Hence, the sum of errors using $g'$ to predict $y$ must be greater than   $|C'|\mathcal{L}_{\mathbb{R}}(\epsilon^*)$.  

Therefore, we have

\begin{eqnarray*}
      \sum_{j=1}^k\sum_{i \in C_j} \mathcal{L}_{\mathbb{R}}( y_i-g_j(x_i))\leq |C'|\mathcal{L}_{\mathbb{R}}(\epsilon^*) < \sum_{i \in C'} \mathcal{L}_{\mathbb{R}}( y_i-g'(x_i)).
\end{eqnarray*}

Hence, $\sum_{i \in C'} \mathcal{L}_{\mathbb{R}}( y_i-g'(x_i)) \geq  \sum_{j=1}^k\sum_{i \in C_j} \mathcal{L}_{\mathbb{R}}( y_i-g_j(x_i))$.
\end{proof}
 
We have the next theorem to state that the maximal homogeneous partition has the lowest encoding bits compared to any other MRC partitions.
\begin{theorem}
\label{theo:MaxHomMDL}
Let $D$ be a set of individual indices, $\mathcal{C}$ be a multi-resolution cluster set , $S_\mathcal{C}$ be an MRB set, $H$ be a linear function class, and $\mathcal{C}^*$ be the maximal homogeneous partition of $S_\mathcal{C}$. For any $\mathcal{C}'$ that is an MRC partition, $\mathcal{L}(\mathcal{C}^*,H)\leq\mathcal{L}(\mathcal{C}',H)$.
\end{theorem}
\begin{proof}

Suppose $\mathcal{L}(\mathcal{C}^*,H) \neq \mathcal{L}(\mathcal{C}',H)$, because both  $\mathcal{C}^*$ and $\mathcal{C}'$ are MRC partitions, all cluster members in both partitions are also members of MRC set $\mathcal{C}$.  It is obvious that $\mathcal{L}(\mathcal{C}^*,H) \neq \mathcal{L}(\mathcal{C}',H)$ because some clusters between $\mathcal{C}'$ and $\mathcal{C}^*$ are different.\\ 

Let $A=\mathcal{C}^*\cap\mathcal{C}' \subset \mathcal{C}$. There are some regions $\mathcal{C}^*-A$ that make  both partitions different. There are two cases that we need to consider. \\

Case 1: suppose $C \in \mathcal{C}^*-A$ and $C'=\{C'_1,\dots,C'_k\} \subseteq \mathcal{C}'-A$ where $\bigcup_{C'_i \in C'} C'_i =C$. Since $C$ is a homogeneous cluster, the subsets of $C$ which are $\{C'_1,\dots,C'_k\}$ must be the homogeneous clusters from the same $X\sim P_X$ and $Y=h(X)+\epsilon$. Because $\bigcup_{C_i \in C'} C_i =C$ implies that we estimate the parameters of $h$ from the same region, hence, $\sum_{C'_i \in \mathcal{C}'} \mathcal{L}(S_{C'_i}|H) \approx   \mathcal{L}(S_{C}|H) $.
\begin{eqnarray*}
    \mathcal{L}(\mathcal{C}',H) -  \mathcal{L}(\mathcal{C}^*,H) = \sum_{C'_i \in \mathcal{C}'} \big( \mathcal{L}(S_{C'_i}|H) + \mathcal{L}(H)_{C'_i}\big)  -\mathcal{L}(S_C|H) - \mathcal{L}(H)_{C}
\end{eqnarray*}

Since $H$ is a linear class and all homogeneous clusters in this case share the same $h(X)$, then $\mathcal{L}(H)_{C} \approx \mathcal{L}(H)_{C'_i}$. 

\begin{eqnarray*}
    \mathcal{L}(\mathcal{C}',H) -  \mathcal{L}(\mathcal{C}^*,H)  =  (|\mathcal{C}'| -1) \mathcal{L}(H)_C >0
\end{eqnarray*}

Case 2: suppose  $C=\{C_1,\dots,C_k\} \subseteq \mathcal{C}^*-A$ and $C' \in \mathcal{C}'-A$ where $\bigcup_{C_i \in C} C_i =C'$. Since $\mathcal{C}^*$ is a maximal partition, $C'$ cannot be a homogeneous cluster. In fact, different homogeneous cluster might have different function $h^*$ where $Y_i = h^*(X_i)+\epsilon$. Let $g'$ be a model that parameters are approximated by $C'$ and $g_1\dots,g_k$ be models that parameters are approximated by $C_1,\dots,C_k$ respectively. Again, we assume that $\mathcal{L}(H)_{C'} \approx \mathcal{L}(H)_{C_i}$.

\begin{eqnarray*}
    \mathcal{L}(\mathcal{C}',H) -  \mathcal{L}(\mathcal{C}^*,H) = \mathcal{L}(S_{C'}|H) + \mathcal{L}(H)_{C'} - \sum_{C_i \in \mathcal{C}} ( \mathcal{L}(S_{C_i}|H) + \mathcal{L}(H)_{C_i}) \\
                 =  \sum_{i \in C'} \mathcal{L}_{\mathbb{R}}( y_i-g'(x_i)) - \sum_{j=1}^k\sum_{i \in C_j}\Big( \mathcal{L}_{\mathbb{R}}( y_i-g(x_i)) \Big)  - (|\mathcal{C}'| -1) \mathcal{L}(H)_{C'} 
\end{eqnarray*}
 Typically, the number of parameters of linear functions is less than the number of individuals. Hence, it is safe to assume that $(|\mathcal{C}'| -1) \mathcal{L}(H)_{C'} \ll  \sum_{i \in C'} \mathcal{L}_{\mathbb{R}}( y_i-g'(x_i))$.  According to Proposition~\ref{prop:homoAndHeteroCls}, we have $\sum_{i \in C'} \mathcal{L}_{\mathbb{R}}( y_i-g'(x_i)) \geq \sum_{j=1}^k\sum_{i \in C_j} \mathcal{L}_{\mathbb{R}}( y_i-g_j(x_i))$. Hence,  
 
 \begin{eqnarray*}
    \mathcal{L}(\mathcal{C}',h) -  \mathcal{L}(\mathcal{C}^*,h) =  \sum_{i \in C'} \mathcal{L}_{\mathbb{R}}( y_i-g'(x_i)) - \sum_{j=1}^k\sum_{i \in C_j}\Big( \mathcal{L}_{\mathbb{R}}( y_i-g_j(x_i)) \Big) \geq 0
\end{eqnarray*}

In both cases,  $ \mathcal{L}(\mathcal{C}',h) - \mathcal{L}(\mathcal{C}^*,h)  \geq 0$. Therefore, $\mathcal{L}(\mathcal{C}^*,h) \leq \mathcal{L}(\mathcal{C}',h)$. 

\end{proof}
\section{Methods}
In this section, we provide the details of method that is used to solve Problem~\ref{prob2}. Fig.~\ref{fig:mainFramework} illustrates the overview of our proposed framework. In the first step, the framework performs the parameter estimation (Section~\ref{sec:paraest}). Then, we use MDL (Section~\ref{sec:mdl}) and cross-validation scheme (Section~\ref{sec:crossv}) to infer the maximal homogeneous partition in Section~\ref{sec:inferMaxPar}. We provide the time complexity of our method in Section~\ref{sec:timeComplex}.

\subsection{Parameter estimation}
\label{sec:paraest}
In linear models, to minimize the magnitude of residuals of prediction, the optimal $h^*$, which minimizes the residuals, can be inferred using the least squares technique.  Given $H$ is a linear function class, we have an optimization problem using least square loss below.

\begin{equation}
    h^* =\argmin_{h\in H} \sum_{i=1}^n (y_i-h(x_i))^2  
    \label{func:leastsquare}
\end{equation}
The linear function is in the form:

\begin{equation}
    h(x_i) = \vec{x}_i^\text{T}\vec{\beta} + \epsilon.
\end{equation}
Where $\vec{x}_i$ is a vector of $x_i$ s.t. $\vec{x}_i(1)=1$ and $\vec{x}_i(j)=x_i(j-1)$, and $\epsilon \sim \mathcal{N}(0,\tau^*)$ is a noisy value where $\tau^*$ is an unknown variance.
Hence, 
\begin{equation}
    \hat{\beta} =\argmin_{\vec{\beta} \in \mathbb{R}^{d+1} } \sum_{i=1}^n (y_i-\vec{x}_i^\text{T}\vec{\beta} )^2  
    \label{func:Linleastsquare}
\end{equation}

We call $\hat{\beta}$ an ordinary-least-squares (OLS) estimator of $\vec{\beta}$. Suppose all columns in $\textbf{X}$ are linear independent, the optimal solution for Eq.~\ref{func:Linleastsquare}  can be solved by the closed-form equation as follows:

\begin{equation}
    \hat{\beta} =(\textbf{X}_1^\text{T}\textbf{X}_1)^{-1}\textbf{X}_1^\text{T}\vec{y}.  
    \label{func:LinleastsquareOpt}
\end{equation}
Where $\textbf{X}_1$ is an $n\times d+1$ matrix s.t. the 1st column of $\textbf{X}_1$ contains only ones, and for any other column $j>1$, $\textbf{X}_1(i,j)=\textbf{X}(i,j-1)$. 

In this step, for each cluster  $C\in \mathcal{C}$, we estimate $\hat{\beta}_{C}$ using data points in cluster $C$ ($(x_i,y_i)$ s.t. $i \in C$ ).

\subsection{Minimum-description-length measure of models}
\label{sec:mdl}
To estimate $\mathcal{L}(\mathcal{C}',H)$ in Eq.~\ref{func:MDLcls}, we need to compute $\mathcal{L}(H)_{C}$ in Eq.~\ref{func:MDLcls2} and $\mathcal{L}(C|H)$ in Eq.~\ref{func:MDLcls3} that both functions require the estimation of $\mathcal{L}_{\mathbb{R}}(x)$. 

Suppose we have two integer numbers $y_1$ and $y_2$ where $|y_1| \ll |y_2|$. When we compress $y_1,y_2$ into a computer memory space, because the magnitude of $y_1$ is a lot smaller than $y_2$, if we need $q$ bits to represent $y_2$, then we can approximately represent $y_1$ without using all $q$ bits.

For integer numbers, we need at least $\mathcal{L}_{\mathbb{I}}(y)=\lceil\text{log}_2 |y|\rceil+1$ bits to represent $y$ in a binary representation where $+1$ bit is for the sign representation of $y$. For a real number representation, the work by ~\cite{Lindstrom:2018:UCR:3190339.3190344} provided the framework to implement the universal coding for real numbers s.t. for any $y \in \mathbb{R}$, $|\mathcal{E}(y)| \propto \lceil\text{log}_2 |y|\rceil$ where $|\mathcal{E}(y)|$ is a number of bits of real-number representation in \cite{Lindstrom:2018:UCR:3190339.3190344}'s work. Hence, for simplicity, for any real number $y$, we can approximately calculate $\mathcal{L}_{\mathbb{R}}(y)$ below.

\begin{equation}
  \mathcal{L}_{\mathbb{R}}(y) = \left\{
  \begin{array}{@{}ll@{}}
	 \lceil\text{log}_2 |y|\rceil+1, & |y|\geq 1 \\
    1, & \text{otherwise}
  \end{array}\right.
\label{func:RealRepFunc}
\end{equation}

If $\textbf{Y}$ is a vector or matrix, then $\mathcal{L}_{\mathbb{R}}(\textbf{Y})=\sum_{i}\mathcal{L}_{\mathbb{R}}(y_i)$ where $y_i$ is a real number element in $\textbf{Y}$.

\subsubsection{Comparing two linear models in the same data points}
Given $H$ is a linear function class and $H_1,H_2 \subset H$ where $H_1 \neq H_2$. To compare $H_1$ and $H_2$ using $\mathcal{L}(\mathcal{C}',H)$ in Eq.~\ref{func:MDLcls} that their parameters are estimated from the same set of data points and the same set of clusters $\mathcal{C}'$, as well as $\sum_{i \in C} \mathcal{L}_{\mathbb{R}}(x_i)$ in both $\mathcal{L}(C|H_1)$ and $\mathcal{L}(C|H_2)$ are the same, we have a following relation.

\begin{equation*}
  \mathcal{L}(\mathcal{C}',H_1) - \mathcal{L}(\mathcal{C}',H_2)  \propto  \sum_{C \in \mathcal{C}'} \big( \sum_{i \in C} \mathcal{L}_{\mathbb{R}}( y_i- h^C_1(x_i) )  - \sum_{i \in C} \mathcal{L}_{\mathbb{R}}( y_i- h^C_2(x_i)  ) + (\mathcal{L}(H_1)_{C} - \mathcal{L}(H_2)_{C} )  \big).
\end{equation*}

Where $h^C_1(x_i)\in H_1$ is a linear function that its $\hat{\beta}$ is estimated using data points in $C$, the same applies for  $h^C_2(x_i)\in H_2$.  The normalization of the equation above is as follows:


\begin{equation}
  \mathcal{I}(\mathcal{C}',H_1,H_2)  = \frac{ \sum_{C \in \mathcal{C}'} \big( \sum_{i \in C} \mathcal{L}_{\mathbb{R}}( y_i- h^C_1(x_i) )  - \sum_{i \in C} \mathcal{L}_{\mathbb{R}}( y_i- h^C_2(x_i)  ) + (\mathcal{L}(H_1)_{C} - \mathcal{L}(H_2)_{C} )  \big)}{\sum_{C \in \mathcal{C}'} \big( \sum_{i \in C} \mathcal{L}_{\mathbb{R}}( y_i- h^C_1(x_i) ) +\mathcal{L}(H_1)_{C} \big)}.
  \label{eq:InfoReRatio}
\end{equation}

 Where $\mathcal{I}(\mathcal{C}',H_1,H_2) \in [-\infty,1]$, and $\mathcal{L}(H_1)_{C},\mathcal{L}(H_2)_{C}$ are the numbers of bits needed for the compression of parameters of a function that is approximated by data points in $C$ using $H_1$ and $H_2$ respectively. We call $\mathcal{I}(\mathcal{C}',H_1,H_2)$ ``Model Information Reduction Ratio''. The function $\mathcal{I}(\mathcal{C}',H_1,H_2)$ represents how well we can reduce the size of our data using $H_2$ instead of $H_1$. For example, if $\mathcal{I}(\mathcal{C}',H_1,H_2) = 0.5$, it implies that $H_2$ reduces the size of space we need to encode the data at $50\%$ compared to the size we encode the same data using $H_1$. In contrast, if $\mathcal{I}(\mathcal{C}',H_1,H_2) = -0.4$, it means using $H_1$ is still a better option since $H_2$ increases the size of space we need to encode the same data around $40\%$. 
 
\subsubsection{Comparing two sets of clusters in the same data points} 
Given $H$ is a linear function class, and $\mathcal{C}'_1,\mathcal{C}'_2$ are two sets of clusters where they cover the same data points or $\bigcup_{C \in \mathcal{C}'_1} = \bigcup_{C \in \mathcal{C}'_2}$. Assuming that $\mathcal{L}(H)_{C}$ for each cluster have the same size $c_H$, we can have ``Cluster Information Reduction Ratio '' $\mathcal{I}(\mathcal{C}'_1,\mathcal{C}'_2,H)$ below:

\begin{small}

\begin{equation}
  \mathcal{I}(\mathcal{C}'_1,\mathcal{C}'_2,H)  = \frac{  \sum_{C_1 \in \mathcal{C}'_1} \big( \sum_{i \in C_1} \mathcal{L}_{\mathbb{R}}( y_i- h^{C_{1}}(x_i) ) \big) - \sum_{C_2 \in \mathcal{C}'_2} \big( \sum_{i \in C_2} \mathcal{L}_{\mathbb{R}}( y_i- h^{C_{2}}(x_i) ) \big) +  c_H(|\mathcal{C}'_1| - |\mathcal{C}'_2|)     }{ \sum_{C_1 \in \mathcal{C}'_1} \big( \sum_{i \in C_1} \mathcal{L}_{\mathbb{R}}( y_i- h^{C_{1}}(x_i) ) \big)+ c_H|\mathcal{C}'_1|}.
\end{equation}
\end{small}
Where $h^{C_{1}},h^{C_{2}}$ are functions that their parameters are approximated using data points in cluster $C_1$ and $C_2$ respectively.  The function $\mathcal{I}(\mathcal{C}'_1,\mathcal{C}'_2,H) \in [-\infty,1]$. $\mathcal{I}(\mathcal{C}'_1,\mathcal{C}'_2,H)$ represents how well the different sets of clusters can encode the same information. For example, if $\mathcal{I}(\mathcal{C}'_1,\mathcal{C}'_2,H)=0.5$, it implies that $\mathcal{C}_2$ reduces the size of space we need to encode the data at $50\%$ compared to the size we can encode the same data using $\mathcal{C}_1$. In contrast, if $\mathcal{I}(\mathcal{C}'_1,\mathcal{C}'_2,H) = -0.4$, it means using $\mathcal{C}_1$ is still a better option since $\mathcal{C}_2$ increases the size of space we need to encode for the same data around $40\%$. 

\subsection{Homogeneity measure of clusters}
\label{sec:crossv}
In order to infer the degree of homogeneity of clusters, we deploy the cross-validation concept. It is typically used as a methodology to measure the generalization of model. If the model is generalized well, then it is capable of performing well in any given dataset outside of a training data that is used to approximate the model parameters~\cite{abu2012learning}.  Given a cluster $C$ and its subsets $\{C_1,\dots,C_k\}$, we use the squared correlation between predicted and real $y$  of cross validation among $C$ subsets as a measure of homogeneity of cluster $C$. 

\begin{equation}
  \eta(C) = \frac{\sum_{C' \subset C} \big(corr(y_{C'},h_{C-C'}(x_{C'}))\big)^2}{|\{C': C' \subset C\}|}
  \label{eq:Rsq}
\end{equation}
Where $h_{C-C'}$ is a function that its parameters are approximated using data points in $C-C'$, and $corr(y_{C'},h_{C-C'}(x_{C'}))$ is a correlation between true values of $y$ in $C'$ and the predicted values of $y$ using $x_C'$ in $C'$ by $h_{C-C'}$. 

We estimate $h_{C-C'}$'s parameters using data points from the rest of subset clusters in $C$ except $C'$, then using $h_{C-C'}$ to predict $y$ using data points of $x$ in $C'$. If $C$ has no subsets in a multi-resolution cluster set $\mathcal{C}$, then we can use $10$-fold cross validation scheme to generate subsets of $C$ by uniformly  separating members of $C$ into 10 subsets that each has equal size. 

Note that the Eq.~10 in~\cite{BROWNE2000108}, which is used as a cross validation evaluation, is similar to $\big(corr(y_{C'},h_{C-C'}(x_{C'}))\big)^2$ in Eq.~\ref{eq:Rsq}. See~\cite{BROWNE2000108} for more details regarding other indices that can be used to evaluate a linear model in a cross validation scheme.

Next, we explore a theoretical property between $\eta(C)$ and homogeneity of clusters.

\begin{definition}[$\gamma$-correlation-linear-learnable dependency relation]
\label{def:eqvrelation}
Given a function a set of data points $S=\{(x_1,y_1),\dots,(x_n,y_n)\}$ where  $x_i\in \mathbb{R}^l$ is a realization of $X\sim P_X$, $y_i$ is a realization of $Y$ where $Y=h(X)$ s.t. $h(X)$ is a linear function, and a threshold $\gamma\in [0,1] $. Let $S_1$ and $S_2$ be disjoint subsets of $S$. Assuming that all dimensions in $x_i$ are linearly independent. 

We say that $S_1$ and $S_2$ are $\gamma$-correlation-linear-learnable dependent if the following conditions are satisfied

1) There exists the ordinary least square estimator $\vec{\beta}$ of the linear function $g(x)=\vec{x}^\text{T}\vec{\beta}$ s.t. $\vec{\beta}$ is estimated using all data points from $S_1$ and $corr(  y_{S_2}, g(x_{S_2})  ) \geq \gamma$ where $corr(   y_{S_2}, g(x_{S_2}) )$ is a correlation between the true values of $y_i$ in $S_2$ and predicted values of $y_i$ using $g(x)$ and $x_i$ in $S_2$. 

2) The same must be true if we use $S_2$ for training $g'$ function. We have $corr(   y_{S_1}, g'(x_{S_1} ) ) \geq \gamma$.

We denote $S_1 \approxeq_\gamma  S_2$.
\end{definition}

\begin{lemma}
A relation in Def.~\ref{def:eqvrelation} is a dependency relation~\cite{AALBERSBERG19881}, which possesses the  reflexive and symmetric properties.
\end{lemma}
\begin{proof}
First, we show that $S_1 \approxeq_\gamma S_1$.  The Theorem 4 in~\cite{turney1994theory} states that if all dimensions in $x$ are linearly independent, then, using OLS and data points in $S_1$ to train a function $g$, $y_i=g(x_i)$ for all $(x_i,y_i)$ in $S_1$. This implies $corr(   y_{S_1}, g(x_{S_1} ) ) = 1 $. Hence, the relation in Def.~\ref{def:eqvrelation} is reflexive.
For the symmetric property, for any $S_1 \approxeq_\gamma S_2$, it implies $S_2 \approxeq_\gamma S_1$ by definition.
\end{proof}

By using $\eta(C)$ in Eq.~\ref{eq:Rsq}, we can define the version of homogeneous cluster with the homogeneous degree property.

\begin{definition}[$\gamma$-Homogeneous cluster]
Given a set of individual indices $D=\{1,\dots,n\}$, and a set of population data $S=\{(x_1,y_1),\dots,(x_n,y_n)\}$ where $x_i \in \mathbb{R}^d$ and $y_i\in\mathbb{R}$. A cluster $C \subseteq D$ is a homogeneous cluster of $S$ if $\eta(C) \geq \gamma$.
\end{definition}

\begin{proposition}
\label{prop:homoRsq}
Given $C$ as a homogeneous cluster and a set of disjoint subsets $\mathcal{C}=\{C_1,\dots,C_m\}$ where $C_i \subset C$ and $\bigcup_{C_i \in \mathcal{C}} C_i = C$. For any $C_i \in \mathcal{C}$, suppose $S_{C_i},S_{C-C_i}$ are sets of data points in $C_i$ and $C-C_i$ respectively. If $S_{C_i} \approxeq_\gamma S'_{C-C_i}$ for any $C_i \in \mathcal{C}$, then $\eta(C)\geq \gamma$, which implies $C$ is a $\gamma$-homogeneous cluster.  
\end{proposition}
\begin{proof}
Since $C$ is a homogeneous cluster, all data points $x\sim X$ in $C$ are generated from the same distribution $X\sim P_X$ and the relation $Y=h(X)$ is the same for $Y$ and $X$ random variables in $C$. For any $C_i \in \mathcal{C}$, because $S_{C_i} \approxeq_\gamma S'_{C-C_i}$, it implies that $corr(  y_{C_i}, g_{C-C_i} (x_{C_i} )  ) \geq \gamma$, where $g_{C-C_i}$ is a function trained by data points in $C-C_i$, and $ y_{C_i}, x_{C_i}$ are data points in $C_i$. Therefore, by averaging all $corr(  y_{C_i}, g_{C-C_i} (x_{C_i} )  )$ from all $C_i$, $\eta(C)\geq \gamma$.
\end{proof}

The Proposition~\ref{prop:homoRsq} suggests that if we know that all pairs of subsets $C_i$ and $C-C_i$ of $C$ have a $\gamma$-correlation-linear-learnable dependency relation, we can imply that $C$ is a $\gamma$-homogeneous cluster. The dependency relation in Def.~\ref{def:eqvrelation} can be checked by running the real data.

\subsection{Inferring maximal homogeneous partition}
\label{sec:inferMaxPar}
 In this section, we propose the algorithm to solve Problem~\ref{prob2}. To determine whether linear models provide informative solution, we compare a linear model of each cluster $C$ with the null model, which is the model that we predict $y$ using the average of $y$, $\bar{y}=\frac{\sum_{i\in C} y_i}{|C|}$. We use Eq.~\ref{eq:InfoReRatio} (Model Information Reduction Ratio~$\mathcal{I}(\mathcal{C}',H_1,H_2)$) to compare linear and null models. We set $\mathcal{C}'=\{C\}$, $H_1=\{h(x)=\bar{y}\}$, $H_2=\{h(x)=\vec{x}_i^\text{T}\vec{\beta}_C\}$ where $\vec{\beta}_C$ is estimated using data points in $C$. Hence, $\mathcal{L}(H_1)_{C} = \mathcal{L}_{\mathbb{R}}(\bar{y})$ and $\mathcal{L}(H_2)_{C} = \mathcal{L}_{\mathbb{R}}(\vec{\beta}_C)$.

 Next, we show that our proposed algorithm (Algorithm~\ref{algo:FindMaxPartiAlgo}) provides the solution for Problem~\ref{prob2}.

\setlength{\intextsep}{0pt}
\IncMargin{1em}
\begin{algorithm2e}
\caption{Finding maximal homogeneous partition in linear function class}
\label{algo:FindMaxPartiAlgo}
\SetKwInOut{Input}{input}\SetKwInOut{Output}{output}
\Input{ A  multi-resolution cluster set $\mathcal{C}=\{C_{j,k}\}$, a set  $\mathcal{S}_\mathcal{C} = \{(x_1,y_1), \dots, (x_n,y_n)\}$, and a homogeneity threshold $\gamma$.  }
\Output{A set of models $\mathcal{H}$, a maximal homogeneous partition $\mathcal{C}^*=\{C\}$, a set of optimal-MDL functions $H^*=\{h^*\}$ of $\mathcal{C}^*$, and the minimum $\eta$ of clusters  $\gamma'$.}
\SetAlgoLined
\nl\For{$C_{j,k} \in \mathcal{C}$ }{
\nl     Calculate $\hat{\beta}_{j,k}$ using Eq.~\ref{func:LinleastsquareOpt} on data points in $C_{j,k}$ \;
 \nl    Append $h(x)_{j,k} = \vec{x}^\text{T}\vec{\beta}_{j,k}$ to $\mathcal{H}$\;
    }

\nl\For{$k \gets 1$ to $n_c$ }{

\nl    \For{$C_{j,k}\in \mathcal{C}$}{
           \textcolor{blue}{\tcp*[h]{$\mathcal{I}(\{C_{j,k} \},H_0,H_{\text{lin}})$ is cached to be used later} }\\
\nl        Calculate $\mathcal{I}(\{C_{j,k} \},H_0,H_{\text{lin}})$ where $H_0=\{h(x)=\bar{y}\}$, $H_{\text{lin}}=\{h(x)\in\mathcal{H}\}$\;
        
\nl        \uIf{there is no $C \in \mathcal{C}^*$ s.t. $C_{j,k}\subset C$}
        {
        \textcolor{blue}{\tcp*[h]{$C_{j,k}$ is not at the last layer} } \\ 
\nl        \uIf{$\exists C \in \mathcal{C}, C \subset C_{j,k}$}
            {
    \nl        Let $\mathcal{C}'=\{C\}$ where $C \subset C_{j,k}$\;
     \nl       Calculate $\mathcal{I}(\mathcal{C}',\{C_{j,k} \},H_{\text{lin}})$\;
    \nl        Append $C_{j,k}$ to $\mathcal{C}^*$ if $\mathcal{I}(\mathcal{C}',\{C_{j,k} \},H_{\text{lin}})>0$ and $\eta(C_{j,k})_{\text{cv}}\geq \gamma$\;
         }
         \textcolor{blue}{\tcp*[h]{$C_{j,k}$ is at the last layer} }\\
     \nl       \Else
            {
    \nl            Append $C_{j,k}$ to $\mathcal{C}^*$\;
            }
        }
    
        }
    }
\nl \For{$C\in \mathcal{C}^*$}{
\nl \uIf{$\mathcal{I}(\{C \},H_0,H_{\text{lin}})>0$}
    {
\nl    Append $h(x)_C \in \mathcal{H}$ to  $H^*$ as an optimal MDL-function of $C$ where $h(x)$ is $C$'s linear function\;
    }
\nl    \Else
    {
\nl    Append $h(x)_C=\bar{y}_C$ to $H^*$ as an optimal MDL-function of $C$ where $\bar{y}_C$ is a mean of $y_i$ in $C$\;
    }
}
\nl Set $\gamma' = \min_{C \in \mathcal{C}^*} \eta(C) $\;
\nl Return $H,\mathcal{C}^*,H^*,\gamma'$\;
\end{algorithm2e}\DecMargin{1em}

\begin{proposition}
Given a  multi-resolution cluster set $\mathcal{C}=\{C_{j,k}\}$, MRB set  $\mathcal{S}_\mathcal{C}$, and the threshold $\gamma=0$.
Algorithm~\ref{algo:FindMaxPartiAlgo} always returns the maximal homogeneous partition.
\end{proposition}
\begin{proof}
For soundness, given $\mathcal{S}_\mathcal{C}$ and $\mathcal{C}=\{C_{j,k}\}$, we prove that Algorithm~\ref{algo:FindMaxPartiAlgo} always provides a maximal homogeneous partition as an output.
According to Theorem~\ref{theo:MaxHomMDL}, the maximal homogeneous partition $\mathcal{C}^*$ always has the $\mathcal{L}(\mathcal{C}^*,H)$ lower  or equal any MRC partition. First, we show that Algorithm~\ref{algo:FindMaxPartiAlgo} provides MRC partition.

For line 4 to 5, the algorithm seeks the clusters from a top layer to a bottom one. It implies that if there exists  a cluster $C\in \mathcal{C}^*$ in the above layer, then it is included to $\mathcal{C}^*$ before its subsets. The condition in line 7 prevents the algorithm to add any subsets of the cluster members of $\mathcal{C}^*$, hence, all clusters in $\mathcal{C}^*$ are disjoint. For some cluster $C$ that is not the last layer member of $\mathcal{C}$, it is either included to $\mathcal{C}^*$ by the line 8-11 or its sub clusters are included in $\mathcal{C}^*$ for some later iteration of the loop.  If there is no sub clusters of $C$ are included until the last layer, then all sub clusters of $C$ at the last layer are included into  $\mathcal{C}^*$ by default. Hence, $\mathcal{C}^*$ covers all individuals because the union of the first layer clusters must cover all individuals and later layer clusters are subsets of some first layer cluster. Hence, the $\mathcal{C}^*$ is MRC partition. 

Second, we show that $\mathcal{C}^*$ is the maximal homogeneous partition. In line 11, we include the cluster $C$ into $\mathcal{C}^*$ only if $\mathcal{I}(\{C_1,\dots,C_k\},\{C \},H_{\text{lin}})>0$.   For the homogeneous cluster $C$ and its subsets $C_1,\dots,C_k$, by Case 1 in  Theorem~\ref{theo:MaxHomMDL}, $\mathcal{I}(\{C_1,\dots,C_k\},\{C \},H_{\text{lin}})>0$.  Since our algorithm performs top-down searching, it  always gets a homogeneous cluster $C$ in the highest possible layers before its subsets by the line 11. Hence, Algorithm~\ref{algo:FindMaxPartiAlgo} provides the maximal homogeneous partition.\\

For completeness, given $\mathcal{C}^*$ as any maximal homogeneous partition, we prove that there is only one possible unique $\mathcal{C}^*$ that Algorithm~\ref{algo:FindMaxPartiAlgo} provides and no other $\mathcal{C}'$ exists. Suppose $\mathcal{C}^*_1,\mathcal{C}^*_2$ are maximal homogeneous partitions of $\mathcal{S}_\mathcal{C}$ and $\mathcal{C}=\{C_{j,k}\}$. Let  $A=\mathcal{C}^*_1 \cap \mathcal{C}^*_2$, there are some clusters in $\mathcal{C}^*_1-A$ that are different from clusters in $\mathcal{C}^*_2-A$. Let assume that $C \in\mathcal{C}^*_1-A$ and $\{C_1,\dots,C_k\} \subseteq \mathcal{C}^*_2-A$ where $\bigcup C_i =C$. According to Case 1 in Theorem~\ref{theo:MaxHomMDL}, because $C$ and its subsets are homogeneous, the length of encoding by $C$ must smaller than using $\{C_1,\dots,C_k\}$. Hence, $\mathcal{C}^*_2$ is not maximal homogeneous partition, which is a contradiction!

Therefore, Algorithm~\ref{algo:FindMaxPartiAlgo} always returns the maximal homogeneous partition which is unique. 
\end{proof}

The $\eta(C_{j,k})_{\text{cv}}\geq \gamma$ condition in line guarantees that all homogeneous clusters are $\gamma$-homogeneous clusters. By setting $\gamma$, Algorithm~\ref{algo:FindMaxPartiAlgo} provides the  maximal homogeneous partition s.t. all homogeneous clusters are $\gamma$-homogeneous clusters. Otherwise, it provides the set of MRC partition that contains the $\gamma$-homogeneous clusters from the highest possible layer that can be found, and the clusters from the last layers for the population that has no $\gamma$-homogeneous cluster in any layer. 

Let $\mathcal{C}'$ be the homogeneous partition with some none-$\gamma$-homogeneous clusters generated by the algorithm at the first time. We can have 

\begin{equation}
    \gamma'=\min_{C' \in \mathcal{C}'} \eta(C').
\end{equation}

After we set the threshold $\gamma = \gamma'$ and run Algorithm~\ref{algo:FindMaxPartiAlgo} for the second time, then the result of the algorithm is the maximal homogeneous partition with all $\gamma'$-homogeneous clusters. This is true since we know that all clusters in $\mathcal{C}'$ are $\gamma'$-homogeneous clusters. By running the algorithm again at the second time, if the result $\mathcal{C}'_2$ is not the same as the first running result $\mathcal{C}'$, by the restriction as line 11, the different parts must be $\gamma'$-homogeneous clusters. However, the experts in the field should determine whether $\gamma'$ is appropriate for their problems.

\subsection{Time complexity}
\label{sec:timeComplex}
The least square approach has the time complexity as $\mathcal{O}(n^2d)$ where $n$ is a number of individuals and $d$ is a number of $X$ dimensions. Given $n_{max}$ is a number of individuals in the largest cluster and $|\mathcal{C}|$ is a total number of clusters from all layers. Algorithm~\ref{algo:FindMaxPartiAlgo} has a time complexity as $\mathcal{O}(|\mathcal{C}|n_{max}^2d)$ where $n_{max}\leq n$. The lower bound is $\Omega(n^2d)$.
\section{Experimental setup}
We use both simulation and real-world datasets to evaluate our method performance. 
\subsection{Simulation data}
\label{sec:simdata}
\begin{figure}
    \centering
    \includegraphics[width=1\columnwidth]{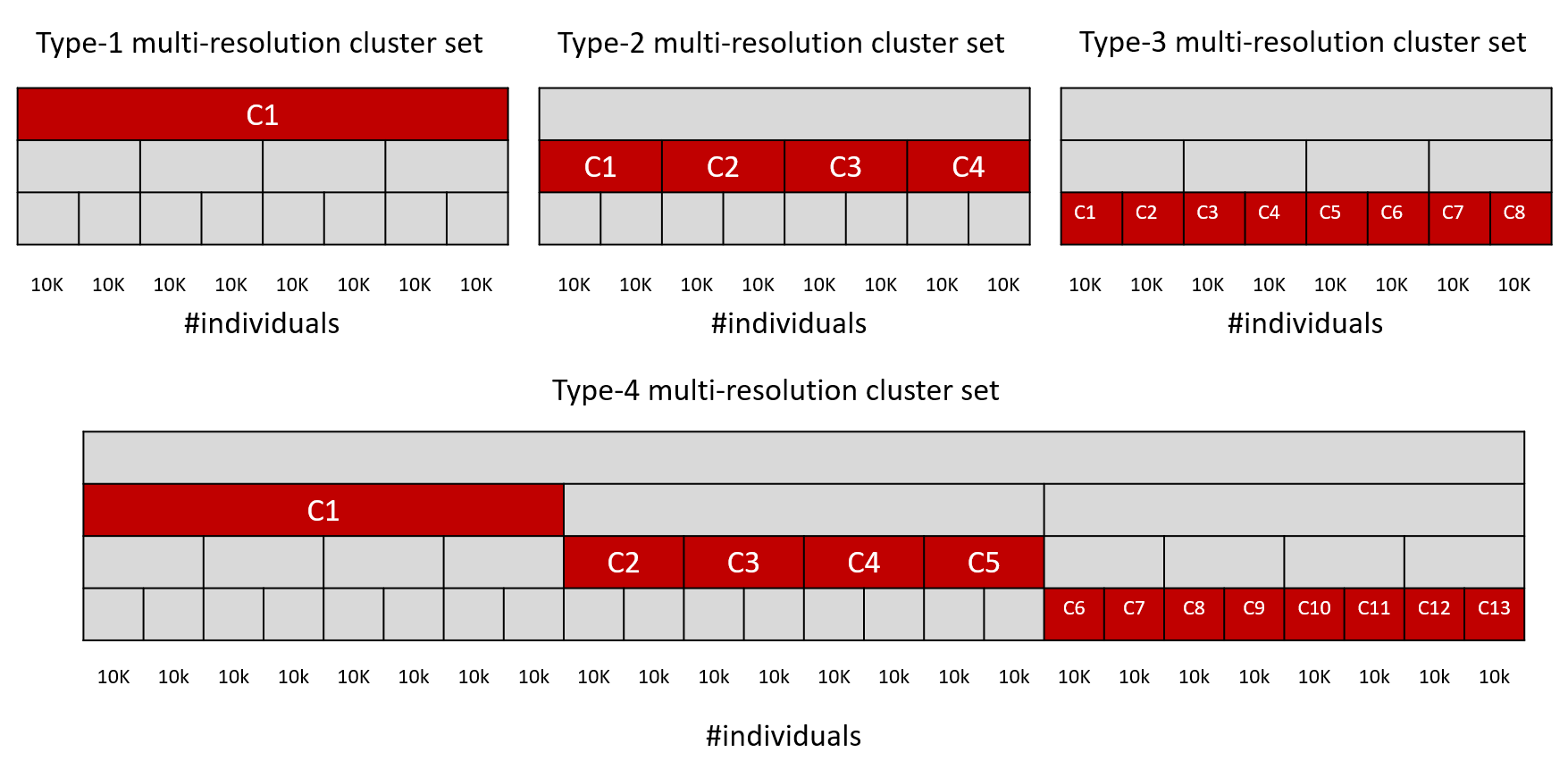}
    \caption{The types of simulation datasets. The $k$th row represents $k$th layer of multi-resolution cluster set. Each $j$th cell in $k$th row represents $j$th cluster of $k$th layer. The red color cells are clusters within  a maximal homogeneous partition $\mathcal{C}^*$. }
    \label{fig:simtypes}
\end{figure}

Given a multi-resolution cluster set $\mathcal{C}=\{C_{j,k}\}$, we assign some subsets of $\mathcal{C}$ to be a maximal homogeneous partition $\mathcal{C}^*=\{C\}$. We generate realizations of independent random variables $X_1,\dots ,X_K$ from normal distribution where $X_k \sim \mathcal{N}(0,1)$. For all members in each cluster $C_j \in \mathcal{C}^*$, their random variables $Y_j,X_j$ have the same joint distribution. We have a following relation between $X_j$ and $Y_j$.

\begin{equation}
    Y_j=c_1\cdot X_j + c_2  
\end{equation}

Where $c_1,c_2 \in \mathbb{R}$, and $\forall k\neq j, Y_j \indep X_k$. This guarantees that each $C_j \in \mathcal{C}^*$ has a different function $Y=h(X)$.

In our simulation setting, we have $X_1,\dots X_{20}$. Each individual $i$ has a pair of values $(y_i,\vec{x}_i)$ where $\vec{x}_i=(x_{1,i},\dots,x_{20,i})$. We have four types of simulation datasets for linear models. Each type has a different  $\mathcal{C}=\{C_{j,k}\}$ and $\mathcal{C}^*=\{C\}$ (see Fig.~\ref{fig:simtypes}). Each cluster in the last layer of any dataset have 10,000 individuals as members. Hence, a type-1/2/3 dataset has $\vec{y}=(y_1,\dots,y_{80000})$, and $80000\times 20$ matrix $\textbf{X}$, while a type-4 dataset has $\vec{y}=(y_1,\dots,y_{240000})$, and $240000\times 20$ matrix $\textbf{X}$.

Moreover, we also generated two types of datasets that has $Y=h(X)$ as a nonlinear function.  The first one is the dataset type of exponential function.

\begin{equation}
    Y_j=c_1\cdot e^{X_j} + c_2  
\end{equation}

The second type is the polynomial function.

\begin{equation}
    Y_j=c_1\cdot {X_j}^d + c_2  
\end{equation}

Where $c_1,c_2 \in \mathbb{R}$, and $\forall k\neq j, Y_j \indep X_k$. We set $d=3$ as the polynomial degree. We generated these nonlinear datasets base on  a multi-resolution cluster set of type-4 dataset in Fig.~\ref{fig:simtypes}. The parameter settings of both types are the same as the type-4 datasets except there are 100 individuals per clusters for the last layer. 

We use these simulation datasets to evaluate whether our framework can infer the correct $\mathcal{C}^*$.  We generated 100 datasets for each type and used them to report the averages of framework performance for each dataset type. We define true positive cases (TP) of prediction as a number of individuals of clusters within the ground-truth maximal homogeneous partition s.t. these clusters are also within the predicted maximal homogeneous partition. The false negative cases (FN) of prediction is a number of individuals in clusters within the ground-truth maximal homogeneous partition that are not in the predicted maximal homogeneous partition. The false positive cases (FP) of prediction is a number of individuals that belong to clusters outside the maximal homogeneous partition but the predicted results claim that these clusters are in the maximal homogeneous partition.  We use TP, FN, and FP to compute precision, recall, and F1-score  values in the result section.

\begin{figure}
    \centering
    \includegraphics[width=1\columnwidth]{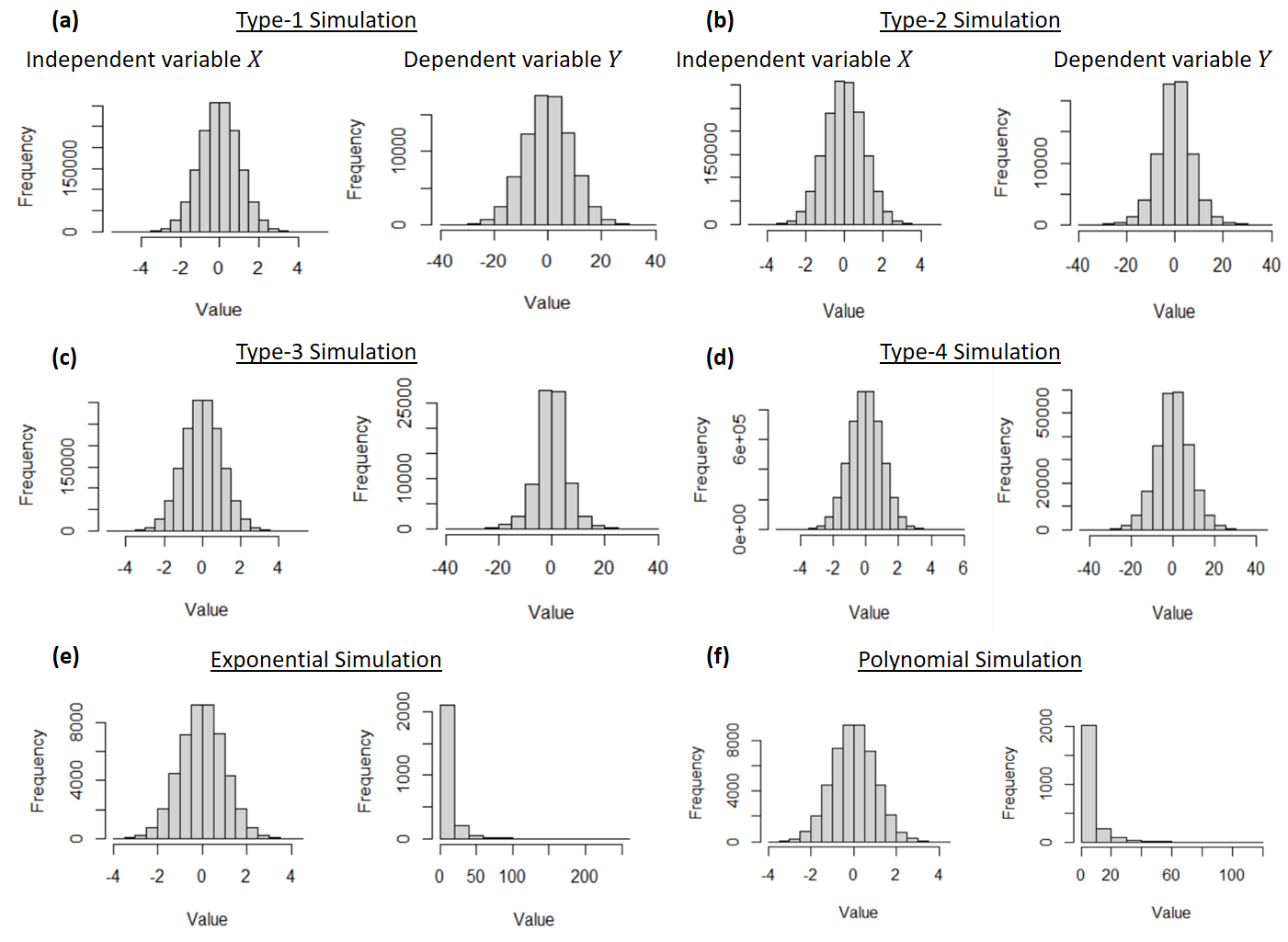}
    \caption{Distributions of simulation data. (a-d) Distributions of independent and dependent variables of simulation type-1 to type-4. (e) Distributions of variables of exponential type. (f) Distributions of variables of polynomial type. }
    \label{fig:SimDataDist}
\end{figure}

Lastly, the distributions of all types of simulation data are shown in Figure~\ref{fig:SimDataDist}.  

\subsubsection{\textbf{Motivation of the simulation design} }
The Multidimensional Poverty Index (MPI)~\cite{alkire2010multidimensional,alkire2019global} is used by United Nations Development Programme (UNDP) to measure acute poverty around the world. By design, MPI considers several factors that can contribute to poverty of people.  In Thailand, there are five main dimensions that government policy makers consider in MPI: health, living condition, education, financial status, and access to the public services (see Table~\ref{tab:ThaiMPI} for more details). Note that MPI typically analyzes in the household level.

Given $Q$ as a matrix derived from surveys where $Q(i,j)$ is a binary value of whether household $i$ fails the dimension $j$ in MPI: one for a fail status and zero for a pass status. In other words, a row $i$ of $Q$ represents a vector of poverty statuses of household $i$.  For example, suppose $j$ is the dimension of household average income, if household $i$ has the average income below the government threshold, then $Q(i,j)=1$.  In the analysis, suppose $i$ fails $d_i$ dimensions. $i$ is considered to be poor if $\bar{d}_i=d_i/n_0$ is below a government threshold where $n_0$ is a number of dimensions. Now, we are ready to define the MPI index.

\begin{equation}
    M_0=q_0\times a_0.
\end{equation}
Where $q_0\in [0,1]$ is a ratio of a number of poor people divided by a number of total population, $a_0$ is the average of $\bar{d}_i$ among poor people, and $M_0\in [0,1]$ is the MPI index. The higher $M_0$ implies higher issues of poverty.

\begin{table}[]
\centering
\caption{The official dimensions of MPI that policy makers of Thailand currently use to design policies that are related to poverty issues.}
\label{tab:ThaiMPI}
\begin{tabular}{|l|l|}
\hline
\rowcolor[HTML]{C0C0C0} 
Main dimensions                                             & Subdimensions                                                                       \\ \hline
                                                            & Birth weight records                                                                \\ \cline{2-2} 
                                                            & \cellcolor[HTML]{EFEFEF}Hygiene \& healthy diet                                     \\ \cline{2-2} 
                                                            & Accessing to necessary medicines                                                    \\ \cline{2-2} 
\multirow{-4}{*}{Health}                                    & \cellcolor[HTML]{EFEFEF}Working out habits                                          \\ \hline
\cellcolor[HTML]{EFEFEF}                                    & Living in a reliable house                                                          \\ \cline{2-2} 
\rowcolor[HTML]{EFEFEF} 
\cellcolor[HTML]{EFEFEF}                                    & Accessing to clean water                                                            \\ \cline{2-2} 
\cellcolor[HTML]{EFEFEF}                                    & Getting enough water for consumption                                                \\ \cline{2-2} 
\rowcolor[HTML]{EFEFEF} 
\multirow{-4}{*}{\cellcolor[HTML]{EFEFEF}Living conditions} & Living in a tidy house                                                              \\ \hline
                                                            & Children as a pre-school age are   prepared for a school                            \\ \cline{2-2} 
                                                            & \cellcolor[HTML]{EFEFEF}Children as a school age  can attend to mandatory education \\ \cline{2-2} 
                                                            & Everyone in household can attend at least   high-school education                   \\ \cline{2-2} 
\multirow{-4}{*}{Education}                                 & \cellcolor[HTML]{EFEFEF}Everyone in household can read                              \\ \hline
\cellcolor[HTML]{EFEFEF}                                    & Adults (age 15-59) have reliable jobs                                               \\ \cline{2-2} 
\rowcolor[HTML]{EFEFEF} 
\cellcolor[HTML]{EFEFEF}                                    & Seniors (age 60+) have incomes                                                      \\ \cline{2-2} 
\multirow{-3}{*}{\cellcolor[HTML]{EFEFEF}Financial status}            & Average income of household members                                                 \\ \hline
                                                            & \cellcolor[HTML]{EFEFEF}Seniors can access public services in need                  \\ \cline{2-2} 
\multirow{-2}{*}{Access to public services}                 & People with disabilities can   access public services in need                       \\ \hline
\end{tabular}
\end{table}

Based on the MPI concept, however, MPI cannot be used to find resolutions of common problems and dependency among dimensions (e.g. income and health issues). Hence, our aim of study is to  identify whether other dimensions (e.g. health, education) affect incomes of households. We are interested in identifying the resolutions of common issues from specific dimensions that affect people's incomes. Given that there are different linear dependencies between synthetic poverty dimensions ($X$ variable) and income ($Y$ variable) in a simulation data in various resolutions, we want to firstly test whether our proposed framework can  identify the correct resolution of models that share common dependencies between income and independent dimensions. Hence, the main purpose of our simulation is to be used as a sanity check whether our framework can correctly identify resolution w.r.t. our ground truth before deploying the proposed framework to the real world datasets.

\subsection{Real-world data: Thailand's population household information}
\label{sec:realdata}

We obtained the dataset of Thailand household-population surveys from Thai government. The surveys were collected in 2019. The surveys were used for estimating the Multidimensional Poverty Index (MPI)~\cite{alkire2010multidimensional,alkire2019global}, which is the main poverty index that UN is currently using. We used the surveys from two provinces, Khon Kaen and Chiang Mai, to perform our analysis. There are 353,910 households  for Khon Kaen, and  378,465 households for Chiang Mai. For each individual household, there are 30 dimensions that describe the characteristics of each household. Each dimension has three possible values: 1 means a good condition, 0 means no data, and -1 means a bad condition. These 30 dimensions can be categorized into five aspects: health, financial status, education, access of public services, and the living conditions. We remove the average household income dimension from the set of independent variable and make it as a dependent variable for prediction.  There are five layers of Thailand administrative divisions: 1) the nation (Thailand), 2) provinces (e.g.  Khon Kaen and Chiang Mai), 3) amphoes, 4)  tambons, and 5) villages. We created a Multi-resolution cluster set from these administrative division. Our task is to infer which layer of these administrative divisions are predictive for income prediction given the household information.

\begin{figure}
    \centering
    \includegraphics[width=1\columnwidth]{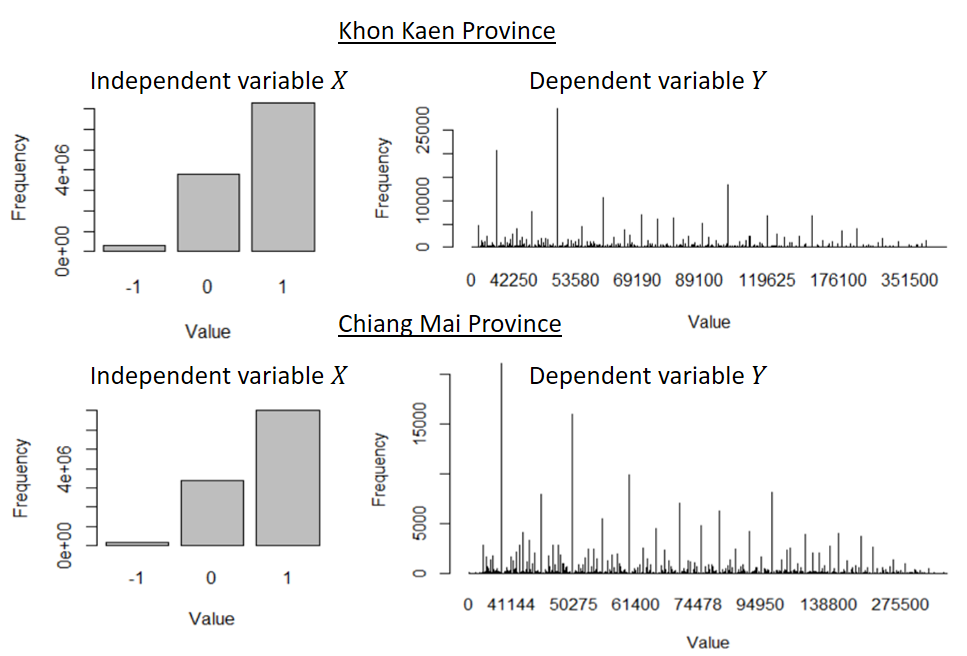}
    \caption{Distributions of real datasets. (top) Distributions of independent and dependent variables of Khon Kaen province. (bottom) Distributions of variables of Chiang Mai province. }
    \label{fig:RealDataDist}
\end{figure}

Lastly, the distributions of real datasets are shown in Figure~\ref{fig:RealDataDist}.  

\subsection{Baseline method: greedy algorithm}
\label{sec:baselineGreedy}

\setlength{\intextsep}{0pt}
\IncMargin{1em}
\begin{algorithm2e}
\caption{Finding a partition using greedy approach}
\label{algo:GreedyPartiAlgo}
\SetKwInOut{Input}{input}\SetKwInOut{Output}{output}
\Input{ A set  $\mathcal{S} = \{(x_1,y_1), \dots, (x_n,y_n)\}$, a  multi-resolution cluster set $\mathcal{C}=\{C_{j,k}\}$  }
\Output{A set of models $\mathcal{H}$, a greedy partition $\mathcal{C}^*=\{C\}$, a set of functions $H^*=\{h^*\}$ of $\mathcal{C}^*$}
\SetAlgoLined
\nl\For{$C_{j,k} \in \mathcal{C}$ }{
\nl     Calculate $\hat{\beta}_{j,k}$ using Eq.~\ref{func:LinleastsquareOpt} on data points in $C_{j,k}$ \;
\nl     Calculate root mean square error (RMSE) for $C_{j,k}$\;
 \nl    Append $h(x)_{j,k} = \vec{x}^\text{T}\vec{\beta}_{j,k}$ to $\mathcal{H}$\;
    }
\nl     Sort clusters in $\mathcal{C}$ by  RMSE and make $\mathcal{C}_S$ as an ascendant sorted list of clusters\;

\nl    \For{$C\in \mathcal{C}_S$}{
        
\nl        \uIf{there is no $C' \in \mathcal{C}^*$ s.t. $C\subset C'$}
        {
\nl         Append $C$ to $\mathcal{C}^*$\;
    
        }
    }
 
\nl \For{$C\in \mathcal{C}^*$}{
\nl \uIf{$\mathcal{I}(\{C \},H_0,H_{\text{lin}})>0$}
    {
\nl    Append $h(x)_C \in \mathcal{H}$ to  $H^*$ where $h(x)$ is $C$'s linear function\;
    }
\nl    \Else
    {
\nl    Append $h(x)_C=\bar{y}_C$ to $H^*$ where $\bar{y}_C$ is a mean of $y_i$ in $C$\;
    }
}

\nl Return $H,\mathcal{C}^*,H^*$\;
\end{algorithm2e}\DecMargin{1em}

To the best of our knowledge, since there is no method that we can compare against our approach directly, we compare our Algorithm~\ref{algo:FindMaxPartiAlgo} performance with the greedy algorithm in Algorithm~\ref{algo:GreedyPartiAlgo} as a main baseline. The greedy algorithm infer the partition by greedily selecting the clusters that have the lowest root mean square error (RMSE) into the output partition. Hence, no algorithm can provide the partition that has the total RMSE as low as  Algorithm~\ref{algo:GreedyPartiAlgo}. On the line 3 of Algorithm~\ref{algo:GreedyPartiAlgo}, the RMSE of the cluster $C_{j,k}$ is calculated by fitting all data points in $C_{j,k}$ to make a model using linear regression. Then, we use the predicted values of the dependent variable from the inferred model and the true values of the  dependent variable to compute RMSE. In the result section, we will show that our approach in Algorithm~\ref{algo:FindMaxPartiAlgo} provides  the maximal homogeneous partition that has the total RMSE as low as Algorithm~\ref{algo:GreedyPartiAlgo} partition's. 

The least square approach has the time complexity as $\mathcal{O}(n^2d)$ where $n$ is a number of individuals and $d$ is a number of $X$ dimensions. Given $n_{max}$ is a number of individuals in the largest cluster and $|\mathcal{C}|$ is a total number of clusters from all layers. Algorithm~\ref{algo:GreedyPartiAlgo} has a time complexity as $\mathcal{O}(|\mathcal{C}|n_{max}^2d)$ where $n_{max}\leq n$.  The time complexity is dominated by the code line 1-4 in Algorithm~\ref{algo:GreedyPartiAlgo} that infers the model from data using the least square approach, which also occurs in Algorithm~\ref{algo:FindMaxPartiAlgo}. This makes both Algorithm~\ref{algo:FindMaxPartiAlgo} and Algorithm~\ref{algo:GreedyPartiAlgo} have the same time complexity. The lower bound is $\Omega(n^2d)$.

\subsection{Baseline method: finite mixtures of regression models}
\label{sec:baselineMixtureModel}

We compare our approach with the finite mixtures of regression models (MR) developed by ~\cite{grun2007applications,JSSv011i08,grun2006fitting}. The mixture-model implementation is in the R package ``flexmix". We set a number of components of mixture models (parameter $k$) corresponding to a number of homogeneous clusters within each dataset. We use a mixture model to show that even if mixtures of linear regression knows the number of homogeneous clusters, in complicated scenarios, its performance is quite unstable. In contrast, our framework that deploys the simple linear regression with multi-resolution partitions can perform better than mixture models in the same complicated datasets.   

Additionally, given $n$ is a number of individuals and $k=|\mathcal{C}|$ is a number of clusters. Algorithm~\ref{algo:FindMaxPartiAlgo} has a time complexity as $\mathcal{O}(k\times n^2\times d)$ while the time complexity of  mixtures models parameter estimating using Expectation-maximization (EM) algorithm~\cite{dempster1977maximum} (deployed by flexmix) is $\mathcal{O}(k\times t_0\times n\times d )$ where $t_0 \in [1,\infty)$ is a number of time steps that EM algorithm converges w.r.t. some predefined threshold. This makes  mixtures of regression models running time is unknown compared to our approach.

\subsection{Software and hardware used in the analysis}
The  computer  specification  that we used in this experiment is the Lenovo Thinkpad T480s, with  CPU Intel Core i7-8650U 1.9 GHz, and RAM 16 GB.
The software we used to conduct the experiment is the R studio version 1.2.5033 based on the R version 3.6.2.  
The R packages we used in the analysis are igraph~\cite{igraph}, caret~\cite{caret}, and ggplot2~\cite{ggplot2}. All experiments were conducted on Microsoft Window 10. The implementation of our framework is in the form of R package with documentation that can be found at~\cite{SharedLink}.
\section{Results}

\subsection{Simulation}

In this section, we report the results of our analysis from simulation datasets (Section~\ref{sec:simdata}). We compared our approach (Algorithm~\ref{algo:FindMaxPartiAlgo}) with the baseline methods from Section~\ref{sec:baselineGreedy} and Section~~\ref{sec:baselineMixtureModel}.  

\begin{figure}
    \centering
    \includegraphics[width=1\columnwidth]{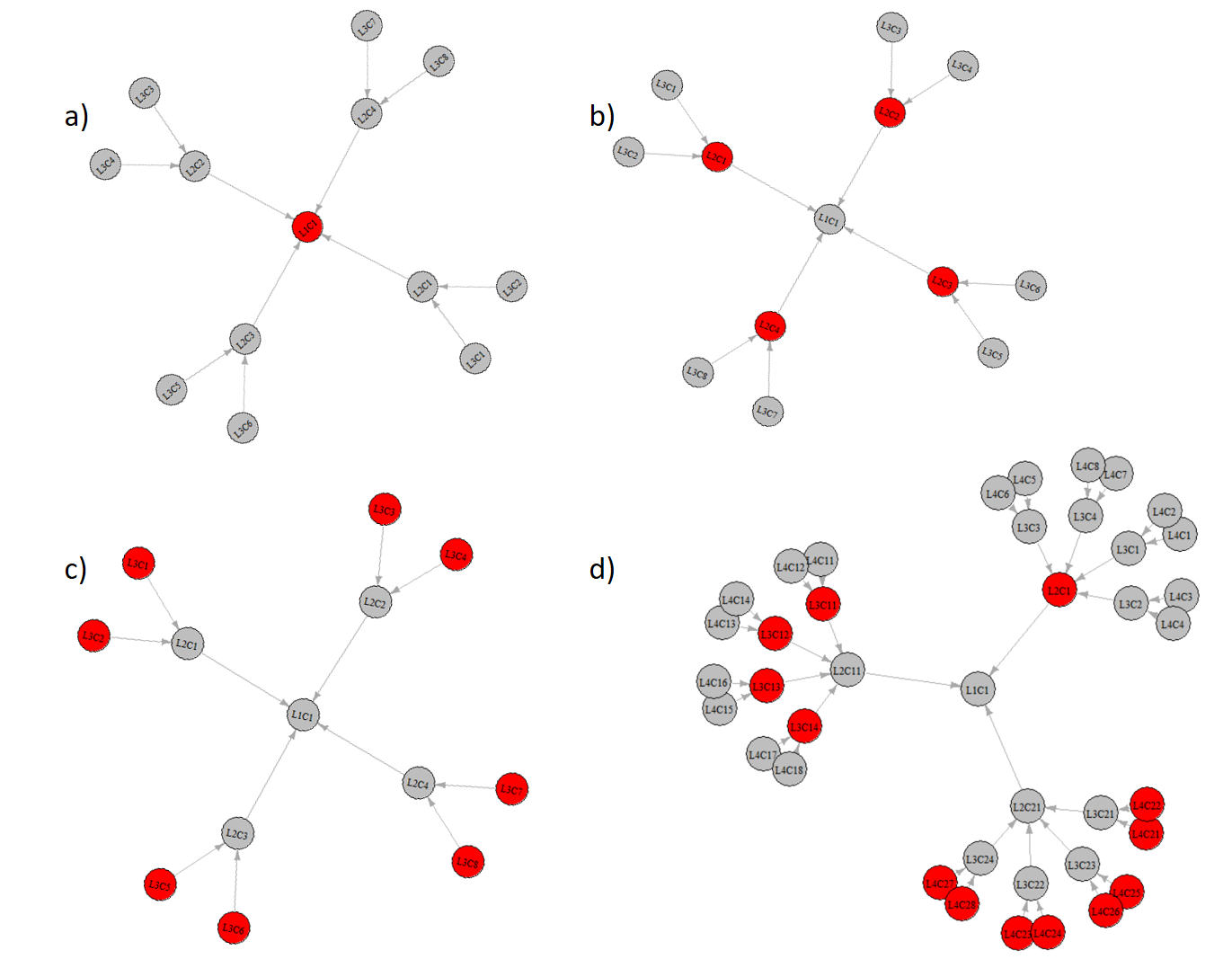}
    \caption{The inferred optimal partition from our method. The nodes represent clusters of population. The edges represent subset relation between a parent cluster (head of arrow) and a child cluster (rear of arrow). The red nodes are clusters selected as members of output partition.}
    \label{fig:resOPTPartition}
\end{figure}

\begin{figure}
    \centering
    \includegraphics[width=1\columnwidth]{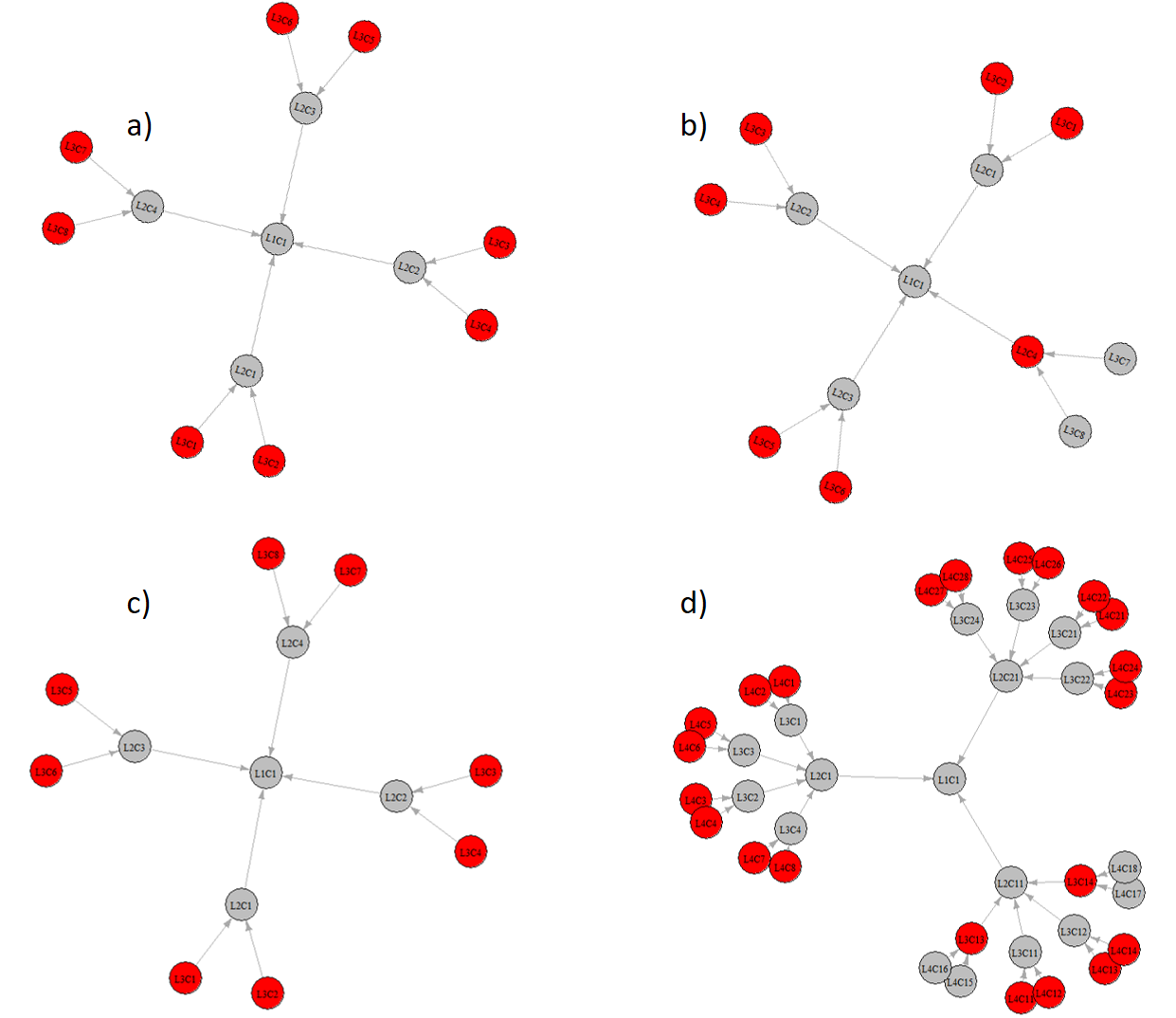}
    \caption{The inferred optimal partition from the greedy approach. The nodes represent clusters of population. The edges represent subset relation between a parent cluster (head of arrow) and a child cluster (rear of arrow). The red nodes are clusters selected as members of output partition.}
    \label{fig:resGreedyPartition}
\end{figure}

The result of output partition from our approach is in Fig.~\ref{fig:resOPTPartition}. The red nodes are clusters selected as members of maximal homogeneous partition. Comparing to the ground truth in Fig.~\ref{fig:simtypes}, our approach can infer the maximal homogeneous partition $100\%$ correctly whether the maximal homogeneous partition consists of clusters at 1st layer (Type-1 datasets), 2nd layer  (Type-2 datasets), 3rd layer  (Type-3 datasets), or multiple layers (type-4 datasets).  In contrast, the result of output partition from the greedy algorithm in Fig.~\ref{fig:resGreedyPartition} shows that it mostly selected the clusters in the last layers as the members of the partition. The greedy algorithm even included the lower clusters that are the subset of homogeneous cluster from the upper layer as the members of output partition. This is because the average RMSE of clusters from lower layer are typically lower than the upper layer clusters. However, the difference between average RMSE of lower and upper layer are not significant (see T1-T3 Datasets in Table~\ref{tb:RMSE}). In fact, the lower-layer clusters trend to have a lower RMSE compared to their super-set cluster because the OLS seems to fit noise easier with less number of individuals. This indicates that the greedy algorithm is sensitive to noise, which implies that it has an over-fitting issue.

\begin{table}[]
\caption{The average RMSE values of approaches on each dataset type. Each element in the table is the average of RMSE from 100 datasets per dataset type. Each row represents results from a specific type of datasets. Each column represents a different method.  }
\label{tb:RMSE}
\begin{tabular}{c|c|c|c|c|c|}
\cline{2-6}
                                 & OPT      & Greedy   & Linear Regression & Mixture of Regression& $y_i-\bar{y}$ \\ \hline
\multicolumn{1}{|l|}{Type-1 Dataset} & $<10^{-6}$ & $<10^{-6}$ & $<10^{-6}$ & $<10^{-6}$       & $6.30$        \\ \hline
\multicolumn{1}{|l|}{Type-2 Dataset} & $<10^{-6}$ & $<10^{-6}$ & $5.30$     & $0.56$        & $6.12$        \\ \hline
\multicolumn{1}{|l|}{Type-3 Dataset} & $<10^{-6}$ & $<10^{-6}$ & $5.94$     & $1.90$    & $6.35$        \\ \hline
\multicolumn{1}{|l|}{Type-4 Dataset} & $<10^{-6}$ & $<10^{-6}$ & $5.58$     & $12.67$      & $6.45$        \\ \hline
\end{tabular}
\end{table}

Table~\ref{tb:RMSE} illustrates the RMSE comparison between different methods.  In all datasets, our approach (OPT) results are the same as the greedy one. In contrast, linear regression perform poorly in all datasets except the Type-1 datasets. This is because the Type-1 datasets have only a single homogeneous cluster, which implies that there is only a single linear function $h(x)$ that generated $y$ from $x$. However, Type-2/3/4 datasets have more than one linear function $h(x)$. By fitting a single linear regression function to these datasets, we expect to have a huge RMSE for these datasets (see Proposition~\ref{prop:homoAndHeteroCls}). For a mixture of regression (Section~\ref{sec:baselineMixtureModel}), it performs well for a dataset that possesses a low number of homogeneous clusters. Nevertheless, for the complicated datasets (Type-4) that has 13 homogeneous clusters in the different layers, the mixture of regression was unable to perform well. This indicates that the mixture of regression, despite of the generalization of linear regression, is quite unstable when the number of homogeneous clusters rises.  The last column is the RMSE from the residual of difference between any $y_i$ with the mean of $y$ that we designate it as a null model. The results demonstrate that OPT and Greedy perform a lot better than the null model.     \\

\begin{table}[]
\caption{The average precision, recall, and F1 score values of approaches on each dataset type. Each element in the table is the average value from 100 datasets per dataset type.}
\label{tb:SimF1}
\begin{tabular}{c|c|c|c|c|c|c|c|c|c|}
\cline{2-10}
                                 & \multicolumn{3}{c|}{OPT} & \multicolumn{3}{c|}{Greedy} & \multicolumn{3}{c|}{Mixture of Regression}\\ \cline{2-10} 
                                 & Precision  & Recall & F1 & Precision  & Recall  & F1   & Precision  & Recall  & F1 \\ \hline
\multicolumn{1}{|c|}{Type-1 Dataset} & 1          & 1      & 1  & 0.01       & 0.01    & 0.01 & 1          & 1      & 1\\ \hline
\multicolumn{1}{|c|}{Type-2 Dataset} & 1          & 1      & 1  & 0.22       & 0.22    & 0.22 & 0.99          & 1      & 0.99\\ \hline
\multicolumn{1}{|c|}{Type-3 Dataset} & 1          & 1      & 1  & 1          & 1       & 1    & 0.96          & 1      & 0.97\\ \hline
\multicolumn{1}{|c|}{Type-4 Dataset} & 1          & 1      & 1  & 0.33       & 0.45    & 0.38 & 0.76          & 0.97      & 0.82\\ \hline
\end{tabular}
\end{table}

In the aspect of performance of inferring the maximal homogeneous partition, we reported the results of methods in Table~\ref{tb:SimF1}. In the table, our method (OPT) can infer maximal homogeneous partition perfectly while the greedy performed mostly poorly except in Type-3 datasets. This is because the greedy algorithm reports mostly the clusters at the last layer to be members of output partition and Type-3 datasets have all clusters members of the maximal homogeneous partition at the last layer.  For a mixture of regression, despite of no access to multi-resolution partitions, it can retrieve homogeneous clusters well for datasets that have few homogeneous clusters (Type-1/2/3). However, it performs fairly in the complicated Type-4 datasets.

\begin{table}[]
\caption{The average RMSE values of approaches on nonlinear datasets: polynomial and exponential datasets. Each element in the table is the average of RMSE from 100 datasets per dataset type. Each row represents results from a specific type of datasets. Each column represents a different method. The (LIN) in rows implies that we used the linear function for all methods except $y_i-\bar{y}$  to fit the data while the (EXP) in rows implies we used exponential function to fit the data except $y_i-\bar{y}$. }
\begin{small}
\begin{tabular}{c|c|c|c|c|c|}
\cline{2-6}
                                         & OPT        & Greedy     & Linear Regression & Mixture of Regression & $y_i-\bar{y}$ \\ \hline
\multicolumn{1}{|l|}{EXP Dataset (LIN)}  & 8.41       & 7.42       & 12.91             & 5.00                  & 13.99         \\ \hline
\multicolumn{1}{|l|}{Poly Dataset (LIN)} & 16.64      & 13.46      & 22.74             & 12.13                 & 24.73         \\ \hline
\multicolumn{1}{|l|}{EXP Dataset (EXP)}  & $<10^{-6}$ & $<10^{-6}$ & 12.26             & 5.16                  & 14.20         \\ \hline
\multicolumn{1}{|l|}{Poly Dataset (EXP)} & 14.59      & 13.45      & 22.43             & 12.40                 & 24.68         \\ \hline
\end{tabular}

\label{tb:RMSE-nonlinear}
\end{small}
\end{table}

 We also extended our approach and other methods to be able to fit data using the exponential function.  We reported the results of using nonlinear datasets in the analysis. Table~\ref{tb:RMSE-nonlinear} shows the average RMSE values of approaches on nonlinear datasets: polynomial and exponential datasets. Our method (OPT) is not the best compared to Greedy and Mixture of Regression methods in the aspect of reducing RMSE. 
 
 The reason that the Mixture of regression outperforms OPT is that it tries to find the fitting models without boundary constraints of partitions, while OPT fits models w.r.t. the given boundary constraints. It is possible that there are some partitions that are not the same as the given partitions from the input that have lower RMSEs when we fit models using them. If such partitions exist, then the Mixture of regression might use the unconstrained partitions to fit its models. Hence, the RMSEs of the Mixture regression can be lower than OPT's in this case. Additionally, the assumption we implicitly made was that if all methods fit models under the partition constraints, then the Greedy has the lowest RMSE. Because the Mixture regression breaks this assumption, its RMSE can be lower than the RMSE of the greedy algorithm.
 
 However, when we consider the average precision, recall, and F1 score values in Table~\ref{tb:SimF1-nonlin} for the task of inferring the correct maximal homogeneous partitions, OPT performed better than all methods. 

Comparing linear and exponential model fitting, in Table~\ref{tb:SimF1-nonlin} (LIN) rows, linear model fitting of OPT performed well in exponential datasets but it did not perform well on the polynomial datasets. In contrast, the result of exponential model fitting in (EXP) rows shows that OPT performed well in both exponential and polynomial datasets. This indicates that our approach is able to be extended to infer maximal homogeneous partitions in nonlinear datasets well. 

\begin{table}[]
\caption{The average precision, recall, and F1 score values of approaches on nonlinear datasets: polynomial and exponential datasets. Each element in the table is the average value from 100 datasets per dataset type. The (LIN) in the rows implies that we used the linear function for all methods  to fit the data while the (EXP) in rows implies we used exponential function to fit the data.}
\label{tb:SimF1-nonlin}
\begin{small}

\begin{tabular}{c|c|c|c|c|c|c|c|c|c|}
\cline{2-10}
                                         & \multicolumn{3}{c|}{OPT}  & \multicolumn{3}{c|}{Greedy} & \multicolumn{3}{c|}{Mixture of Regression} \\ \cline{2-10} 
                                         & Precision & Recall & F1   & Precision  & Recall  & F1   & Precision       & Recall       & F1        \\ \hline
\multicolumn{1}{|c|}{EXP Dataset (LIN)}  & 0.83      & 0.87   & 0.85 & 0.30       & 0.43    & 0.35 & 0.19            & 0.16         & 0.17      \\ \hline
\multicolumn{1}{|c|}{Poly Dataset (LIN)} & 0.53      & 0.57   & 0.55 & 0.30       & 0.43    & 0.35 & 0.16            & 0.16         & 0.16      \\ \hline
\multicolumn{1}{|c|}{EXP Dataset (EXP)}  & 1.00      & 1.00   & 1.00 & 0.33       & 0.44    & 0.38 & 0.19            & 0.18         & 0.19      \\ \hline
\multicolumn{1}{|c|}{Poly Dataset (EXP)} & 0.92      & 0.94   & 0.93 & 0.30       & 0.43    & 0.35 & 0.16            & 0.15         & 0.15      \\ \hline
\end{tabular}
\end{small}
\end{table}

In summary, the results in this section indicate that our approach performance has the same RMSE as the greedy one (Table~\ref{tb:RMSE}), nevertheless, its performance is a lot better in the aspect of inferring the maximal homogeneous partition (Table~\ref{tb:SimF1}). For mixture of regression model, it performed well for Type-1/2/3 datasets, while its performance decreased for a Type-4 datasets, which have clusters in the maximal homogeneous partition from multiple layers.   In fact, the homogeneous clusters that are inferred by a mixture of regression might not be consistent with a given multi-resolution partitions. However, we can deploy mixtures of regression to approximate homogeneous clusters when datasets come without multi-resolution partitions. But we still need our approach to estimate the degree of homogeneity of each cluster as well as using it to compare different ways of partitioning a population. Moreover, we can even use the mixture of regression as a kernel instead of using linear regression to predict our dependent variable.

\subsection{Case study: inferring informative-administrative-level subdivision to predict population incomes}

We use household population data (Section~\ref{sec:realdata}) to demonstrate the application of our framework to help the policy maker to make a policy. In reality, policy maker cannot create the unique policy for every village to solve the poverty issues because of the limited resource. On the other hand, making a single policy cannot solve all poverty issues since each region has their own unique problems. Hence, we propose to use our framework to find the largest level of administrative subdivision that have enough common issues (measuring by $\eta$ in Eq.~\ref{eq:Rsq}) so that the policy makers do not need to establish a policy for each village. In our framework, a $\gamma$-Homogeneous cluster is considered to be an informative subdivision that the policy makers can establish a policy. In this case study, we set $\gamma=0.05$ which requires the correlation between dependent variable $Y$ and predicted $\hat{Y}$ around 0.22 that is closed to a moderate correlation (the correlation around 0.3 is considered as a moderate correlation~\cite{cohen1988set,cohen2013statistical}).

\begin{table}[]
\caption{The RMSE values of approaches on each province population. Each element in the table is the RMSE value. Each row represents results from a specific province. Each column represents a different method.}
\label{tb:RMSErealdata}
\begin{tabular}{c|c|c|c|c|}
\cline{2-5}
\multicolumn{1}{l|}{}                     & \textbf{OPT} & \textbf{Greedy} & \textbf{Linear Regression} & \textbf{$y_i-\bar{y}$} \\ \hline
\multicolumn{1}{|c|}{\textbf{Khon Kaen}}  & 69,620.3      & 67,210           & 82,571                      & 84,994.6              \\ \hline
\multicolumn{1}{|c|}{\textbf{Chiang Mai}} & 89,067.1             &   87,727.5              &   104,938                         &     106,511                 \\ \hline
\end{tabular}
\end{table}

Table~\ref{tb:RMSErealdata} illustrates RMSE of each methods from two provinces: Khon Kaen and Chiang Mai. The result shows that our approach (OPT) has a bit higher RMSE compared to the greedy algorithm. In contrast, linear regression and the null model $y_i-\bar{y}$ have larger RMSE.  Linear regression represents the approach of finding one policy for all regions within a province to predict the income of households, while our approach is trying to find a unique policy for each specific region. The result in Table~\ref{tb:RMSErealdata} indicates that each region has its own problem. This is why fitting linear regression for the entire province population performed poorly.  

\begin{table}[]
\caption{The number of clusters that are the members of output partition in each method.}
\label{tb:Nclusters}
\begin{tabular}{l|c|c|c|c|c|c|c|c|c|}
\cline{2-9}
                                          & \multicolumn{4}{c|}{\textbf{OPT}}                                        & \multicolumn{4}{c|}{\textbf{Greedy}}                                     \\ \cline{2-9} 
                                          & \textbf{1st} & \textbf{2nd} & \textbf{3rd} & \textbf{4th}  & \textbf{1st} & \textbf{2nd} & \textbf{3rd} & \textbf{4th}  \\ \hline
\multicolumn{1}{|c|}{\textbf{Khon Kaen}}            & 0            & 2            & 55           & 1,785            & 0            & 0            & 1            & 2,643         \\ \hline
\multicolumn{1}{|c|}{\textbf{Chiang Mai}} &         0     &      1        &     41         &       1,595       &   0           &    0           &       0       &    2,223          \\ \hline
\end{tabular}
\end{table}

Table~\ref{tb:Nclusters} shows the number of clusters that are the members of output partition in each method. The first layer is the province level, the second layer is the amphoe level, the third layer is the tambon level, and the fourth layer is the village level. The result indicates that our approach can detect informative subdivisions beyond the last layer while the greedy approach can report only the clusters in the last layer. 

After we got the maximal homogeneous partition, we need to validate whether the result is consistent with the ground truth of the government records.\footnote{ All reference data of government records used here are from  \href{https://www.tpmap.in.th/about_en/}{https://www.tpmap.in.th/}. We used the 2019 records of poor people as a ground truth in this result section.}  The following areas are examples of members of the maximal homogeneous partition. 

First, in Khon Kaen province, the ``Subsomboon'' tambon, which is on the 3rd layer, has 11 villages in 4th layer. 8 out of 11 villages have the majority of poor people suffering from the lack of the education issue. Second, another tambon in Khon Kaen province is Khu Kham. It has 7 out of 8 villages faces the lack of education issue among poor people. Third, in Chiang Mai province, in San Sai amphoe (2nd layer) has 12 tambons. 10 out of 12 tambons has financial and education as leading issues. 

We can see that, for each tambon in the examples above, policy makers can make a single policy for all villages because the majority of subareas in each homogeneous area have almost the same issues. 

Next, the non-homogeneous cases are provided below.

In Khon Kaen province, Sila tambon (3rd layer) consists of 28 villages. There are 11 villages that have no poor people. There are 11 villages that have financial issues among the majority of poor people. There are 5 villages facing health issues among the majority of poor people in each village. One village has equal numbers of poor people facing financial and health issues. We can see that the policy makers need to group villages in this tambon w.r.t. the issue types before making policies that fit each type of issue.  

In Chiang Mai province, Chiang Dao amphoe (2nd layer) has 7 tambons. Two tambons have financial issues among the majority of poor people. Two tambons have financial and education issues. Two tambons have financial and health issues. Lastly, one tambon has the education issue as a main problem. It is quite challenging for policy makers to place a single policy for the entire Chiang Dao amphoe. 

The next one is the example of how policy makers can use our system to combat poverty. One of the tambons in the output maximal homogeneous partition of our approach  from Khon Kaen province is ``Subsomboon''.  We used the t-test to find the variable importance of coefficients of regression that have values far from zero for this tambon. The null hypothesis is that the absolute of specific coefficient is zero, while the alternative hypothesis is that the absolute of coefficient is greater than zero. We reject the null hypothesis at $\alpha = 0.01$.  

Out of 30 coefficients, there are only four coefficients that we can reject the null hypothesis. However, the only positive coefficient is the indicator of  whether a household has a saving account with a bank. This implies that having saving account is associated with income. The policy maker should consider why some households can have saving account to make a policy of combating poverty in Subsomboon tambon.  For Chiang Mai province, San Sai amphoe is one of the informative clusters in the output partition. The variables that are important and have positive coefficients mostly are health issues in the household that might make the members of household cannot work efficiently.  Hence, the policy maker should consider to make a policy that supports the health of people in the area.   

In practice, the greedy algorithm performs well only when all homogeneous clusters are mostly in the last layer. Even in this case, its performance is not significantly better than OPT.  Moreover, the time and space complexities of both algorithms are the same. Hence, OPT should always be used to find the maximal homogeneous partition.  
\section{Discussion}
\subsection{Inferring multiresolution partitions}

In practice, the obvious case that a dataset of multiresolution partitions can be found is a dataset that is related to administrative subdivisions of country. In each sub-districts in a specific resolution, the data can be varied, such as household income, types of land use,  distribution of labor force, etc. In the eyes of policy makers, knowing which area share similar models for target properties make them have an easier way to manage resource or to place policies. 

However, in some case, there is no given multiresolution partitions. For example, suppose we divide a natural farming area into a grid and we have information about history of land use of each block in the grid (e.g. the remaining resource). Our goal is to infer partitions of multiple neighbor blocks that we can declare as either preserved areas to protect the land or non-preserved areas. In this situation, we can use hierarchical clustering~\cite{10.1093/comjnl/16.1.30, 10.1093/comjnl/20.4.364, doi:10.1080/01621459.1963.10500845} (hclust function in R programming~\cite{Rprog} ) to build a hierarchical tree of blocks. We can assign layers of partitions based on the distance of nodes from the nearest leaf. Another approach is to use the mixture regression method~\cite{grun2007applications,JSSv011i08,grun2006fitting} in Section~\ref{sec:baselineMixtureModel}. We might run the mixture regression method to get the first-layer result, then, for each inferred partition, we might run the method again to find the sub-districts. Nevertheless, this approach requires users to set an appropriate number of clusters $k$.

\subsection{Extension to non-linear models}
\label{sec:ex2nonlin}
It is possible to extend our work to non-linear models if we care only a number of bits and information we need to explain models. However, in machine learning, there are issues of underfitting and overfitting that we should consider. 

To resolve underfitting issues, we try to find $h$ that  minimizes $\mathcal{L}(S|h)$ in Eq.~\ref{eq:LSh}. Nevertheless, a complex model typically fits data well, which makes $\mathcal{L}(S|h)$ small. To resolve the overfitting issue, since we prefer simpler model that has its performance close to complicated model, we try to find $h$ that also  minimizes $\mathcal{L}(h)$ in Eq.~\ref{eq:Lh}. 

We can say that $\mathcal{L}(h)$ represents the model complexity. However, there are many challenges that we need to address in order to measure $\mathcal{L}(h)$ for any arbitrary $h$.

It is straightforward to compare $\mathcal{L}(h_i)$ against $\mathcal{L}(h_j)$ from the same function class (e.g. a linear function class, an exponential function class, etc) by measuring a number of bits we need to keep coefficients and terms. For example, the model complexity of one linear function that has less number of terms with small-size coefficients can be considered as a simpler model than a complicated linear function that has many terms with large coefficients. On the contrary, it is challenging to compare two models from different classes and declare that one model is less complicated than another (e.g. sine function vs. logarithmic function). 

One of the possible ways to measure the complexity is to utilize the concept of the Vapnik-Chervonenkis (VC) dimension~\cite{doi:10.1137/1116025} in learning theory. The VC dimension of a function class $\mathbb{H}$ can be used to measure a sample complexity~\cite{abu2012learning,mohri2018foundations}; a model with a higher VC dimension requires more data to achieve the same generalization performance compared to a model with a lower VC dimension. In our context, we can give a higher number of  bits to $\mathcal{L}(h)$ for a model from a function class with the higher VC dimension. Nevertheless, not all function classes  have VC dimensions. Alternatively, in learning theory, we can use Rademacher complexity~\cite{koltchinskii2001rademacher} that measures the richness of a function family $\mathbb{H}$ in term of how well $\mathbb{H}$ can fit random noise. The more complex function class can fit random noise better than the simpler function class~\cite{mohri2018foundations}. Hence, we make $\mathcal{L}(h)$ to have more bits if a function class fits random noise better than another function class. 

\subsection{Sparsity regularization}
\label{sec:SparsityReg}

In the context of sparsity regularization, we also want to get the model that has the minimum number of dimensions we used in the model along with the prediction performance $\mathcal{L}(S|h)$ (Eq.~\ref{eq:LSh}) and  the model complexity $\mathcal{L}(h)$ (Eq.~\ref{eq:Lh}). Suppose $\beta=(p_1,\dots,p_l)$ is a vector of parameters for a model $h$, given a set of data $S=\{(x_1,y_1),\dots,(x_n,y_n)\}$ and a penalty weight $\lambda$, we can have the following optimization problem to find the optimal $\beta$ from data that utilizes the sparsity regularization.

\begin{equation*}
    \beta^*(\lambda) = \text{argmin}_{\beta \in \mathbb{R}^l } \frac{1}{n}\sum_{i=1}^n loss (y_i, f_{\beta}(x_i) ) +\lambda\left\| \beta \right\|_0 
\end{equation*}

Where $\left\| \beta \right\|_0$ is the number of parameters in $\beta$ that has nonzero values, $f_{\beta}(x_i)$ is a function that returns predicted value of $y_i$ using $\beta$ and $x_i$, and $loss()$ is a loss function. Based on the optimization problem above, we can modify $\mathcal{L}(h)$ to utilize the sparsity regularization.

\begin{equation*}
\mathcal{L}(h,\lambda) = \sum_{i=1}^l \mathcal{L}_{\mathbb{R}}(p_i)   +  \lambda\mathcal{L}_{\mathbb{R}}(\left\| \beta \right\|_0)
\end{equation*}

Hence, we have

\begin{equation}
\label{eq:Sreg}
\mathcal{L}(S,\lambda)_{\mathbb{H}} =\min_{h\in\mathbb{H}} \mathcal{L}(h,\lambda)+ \mathcal{L}(S|h).
\end{equation}

The Eq.~\ref{eq:Sreg} above encourages the model $h$ that has a good performance (minimizing $\mathcal{L}(S|h)$) as well as using less parameters with the parameters that have small magnitudes (minimizing $\mathcal{L}(h,\lambda)$).

\subsection{Encoding methods for real numbers}
As a reminder, throughout this paper, we have Assumption~\ref{assump:bignumberMoreCode} states as follows. 

For any $x_1,x_2 \in \mathbb{R}$ s.t. $|x_1|\geq|x_2|$, we have 

\begin{equation*}
    \mathcal{L}_{\mathbb{R}}(x_1)\geq \mathcal{L}_{\mathbb{R}}(x_2).
\end{equation*}

This assumption is true for the \textit{Single-Precision Floating-point Format} (the IEEE Standard 754~\cite{hough2019ieee}) that is widely used in several well-known programming languages (e.g. Fortran,  C, C++, C\#, Java) to encode real numbers. However, the floating-point format makes every real number use an equal number of bits: 32 bits per number.

Nevertheless, the work by ~\cite{Lindstrom:2018:UCR:3190339.3190344} provided the framework to implement the universal coding for real numbers. In this coding,  a number of bits varies proportional to a real-number magnitude. Specifically, for any $y \in \mathbb{R}$,

\begin{equation*}|\mathcal{E}(y)| \propto \lceil\text{log}_2 |y|\rceil\end{equation*} 

where $|\mathcal{E}(y)|$ is a number of bits of real-number representation in \cite{Lindstrom:2018:UCR:3190339.3190344}'s work. In practice, if a specific application has high variation of magnitudes of real numbers, then the encoding in ~\cite{Lindstrom:2018:UCR:3190339.3190344} might be able to reduce spaces requires for keeping real numbers and our framework can obviously gain this benefit. For a specific application, one might check whether the encoding scheme for real numbers that complied with our assumption is suitable for being used under the application. 

\section{Conclusion}
In this paper, we addressed the challenges encountered by many policy makers: to find the largest possible areas with common problems where similar policies can be implemented and limited resources are optimally utilized. Given a set of target variable and predictors, as well as a set of multi-resolution clusters that represent administrative divisions as inputs, we proposed a framework to infer a set of largest-informative clusters that have efficient predictors. Efficient predictors of each informative cluster cover the entire cluster population. We used both simulation datasets and real-world dataset of Thailand's population household information to evaluate and to illustrate the application of our framework. The results showed that our framework performed better than all baseline methods. Moreover, in Thailand's population household information, the framework can infer the predictors of people income for the specific areas, which can be used to guide policy makers to utilize this information to combat poverty. Particularly, simulated data shows that our approach can estimate homogeneity of each cluster and compare different ways of population partition, which would enhance the formulation of evidence-based policies as well as their assessments. This finding is confirmed with real-world data from the two major provinces of Thailand where it is possible to cluster regions with similar problems up to the 2nd (Amphoe) level. More importantly, our framework can be used beyond the context of income prediction but for any kind of regression analysis on multi-resolution clusters. Lastly, we also provide the R package, MRReg, which is the implementation of our framework in R language with the manual at \cite{SharedLink}. The official link for MRReg at the Comprehensive R Archive Network (CRAN) can be found at \href{https://cran.r-project.org/package=MRReg}{https://cran.r-project.org/package=MRReg}.

\section*{Acknowledgment}

The authors would like to thank the National Electronics and Computer Technology Center (NECTEC), Thailand, to provide our resources in order to successfully finish this work.

This paper was supported in part by the Thai People Map and Analytics Platform (TPMAP), a joint project between the office of National Economic and Social Development Council (NESDC) and the National Electronic and Computer Technology Center (NECTEC), National Science and Technology Development Agency (NSTDA), Thailand.

\bibliographystyle{ACM-Reference-Format}


\end{document}